\pdfoutput=1
\documentclass[10pt, preprint]{elsarticle}

\usepackage[utf8]{inputenc}
\usepackage{amsmath,amsfonts,stmaryrd,amssymb,bbm} 
\usepackage{amsthm}
\usepackage{mathtools}
\usepackage{tikz}
\usepackage{enumitem}
\usepackage{nicematrix}
\usepackage{booktabs}
\usepackage{hyperref}
\hypersetup{
    colorlinks,
    citecolor=black,
    filecolor=black,
    linkcolor=black,
    urlcolor=black
}
\DeclarePairedDelimiter{\ceil}{\lceil}{\rceil}
\DeclarePairedDelimiter{\floor}{\lfloor}{\rfloor}
\DeclarePairedDelimiter{\norm}{\lVert}{\rVert}

\newcommand{\mC}[0]{{a}} 
\newcommand{\eC}[0]{{ b}} 
\newcommand{\expBaseC}[0]{{\beta}} 
\newcommand{\modulationC}[0]{{\theta}}
\newcommand{\logTe}[0]{{\gamma}} 

\newcommand{\entOfSetLazy}[0]{\mEnt(\epsilon; \mathcal{C}(\mC{}, \eC{}), \rho) }

\newtheorem{lemma}{Lemma}[section]
\newtheorem{theorem}{Theorem}[section]

\newtheorem{remark}{Remark}[section]
\newtheorem{definition}{Definition}[section]

\title{Metric Entropy Limits on Recurrent Neural Network Learning of Linear Dynamical Systems}
\author{Clemens Hutter, Recep G\"ul, and Helmut B\"olcskei}
\begin{document}
\newcommand{\normIII}[1]{{\left\vert\kern-0.25ex\left\vert\kern-0.25ex\left\vert #1 
    \right\vert\kern-0.25ex\right\vert\kern-0.25ex\right\vert}}

\renewcommand{\vec}[1]{#1}

\newcommand{\h}[1]{\vec{h}[{#1}]}
\newcommand{\ind}[1]{\mathbbm{1}_{\{#1\}}}
\newcommand{\eye}[1]{\mathbbm{I}_{#1}}
\newcommand{\Z}[1]{\mathcal{Z}\left\{ #1\right\}}
\newcommand{\iZ}[1]{\mathcal{Z}^{-1}\left\{ #1\right\}}
\newcommand{\R}[0]{\mathbb{R}}
\newcommand{\WW}[0]{W_{\mathbf{a}, \mathbf{b}}}
\newcommand{\Imat}[1]{\mathbb{I}_{#1}}
\newcommand{\Omat}[0]{\mathbb{O}}
\newcommand{\Ovec}[1]{\vec{0}_{#1}}
\newcommand{\Ivec}[1]{\vec{1}_{#1}} 

\newcommand{\quantF}[1]{S_{\delta} (#1)}
\newcommand{\quantSet}[0]{\mathbb{S}_{\delta}}

\newcommand{\mnorm}[1]{\norm{#1}_{max}}
\newcommand*{\vvbar}{\rule[-1ex]{0.5pt}{2.5ex}}
\newcommand*{\hhbar}{\rule[.5ex]{2.5ex}{0.5pt}}
\newcommand{\logqd}[2]{log_2(5 + \frac{#1}{#2})}

\newcommand{\N}[0]{\mathbb{N}}
\newcommand{\C}[0]{\mathbb{C}}

\newcommand{\sysOp}[0]{\mathcal{L}}
\newcommand{\seqSpace}[0]{\ell_\infty} 
\newcommand{\seqSpaceC}[0]{\ell_\infty} %

\newcommand{\rnnOp}[1]{\mathcal{R}_{#1}}
\newcommand{\quantRnnOp}[1]{\widetilde{\rnnOp{#1}}}
\newcommand{\quantRnnImp}[0]{\widetilde{k}}

\newcommand{\lengthK}[0]{T}%

\newcommand{\Hardy}[1]{{\mathcal{H}^{#1}}}

\newcommand{\HardyTwo}[1]{ \sqrt{ \sup_{r<1} \frac{1}{2\pi} \int_{0}^{2\pi} #1 d\theta }}

\newcommand{\mEnt}{\mathcal{E}}
\newcommand{\entSet}{\mathcal{C}}
\newcommand{\sysSubset}{\entSet_{C, a}}
\newcommand{\opMetric}{\rho_{*}}
\newcommand{\hardyRho}{\rho_{H}}

\newcommand{\unitImpulse}{\delta}
\newcommand{\impResp}{k}
\newcommand{\transF}{K}

\newcommand{\opNorm}[1]{\normIII{#1}_{2}}
\newcommand{\hNorm}[1]{\norm{#1}_{\Hardy{\infty}}}
\newcommand{\hhNorm}[1]{\norm{#1}_{\Hardy{2}}}

\newcommand{\freq}[0]{\theta}
\newcommand{\spread}[0]{S_{\mathcal{L}}}
\newcommand{\approxSpread}[0]{\widetilde{S}_{\mathcal{L}}}
\newcommand{\relu}[0]{\rho}

\newcommand{\fconti}[0]{\gamma}

\newcommand{\xHistory}[0]{g}
\newcommand{\periodState}[0]{e}
\newcommand{\periodStateB}[0]{e'}

\newcommand{\periodMatrix}[0]{A_e} 
\newcommand{\periodBias}[0]{b_e} 

\newcommand{\xtPeriod}[2]{\widetilde{x}[#1, #2]} 
\newcommand{\periodizedK}[1]{\widetilde{k}[#1]} 
\newcommand{\kNormalization}[0]{\overline{k}}

\newcommand{\Ar}[0]{A_r}
\newcommand{\Ao}[0]{A_o}
\newcommand{\Ah}[0]{A_h}
\newcommand{\AF}[0]{A_F}

\newcommand{\stateDim}[0]{m} 
\newcommand{\hidLayerDim}[0]{n} 
\newcommand{\repDim}[0]{{R}} 
\newcommand{\maxTime}[0]{D}

\newcommand{\rnnOut}[0]{y}

\newcommand{\xPeriod}[0]{\widetilde{x}}

\newcommand{\hHist}[0]{\dot{h}}
\newcommand{\hPS}[0]{\ddot{h}}
\newcommand{\gHist}[0]{\dot{g}}

\newcommand{\modulatedX}[0]{\widehat{x}}

\newcommand{\hidElm}[0]{\mathring{h}}
\newcommand{\stateDimElm}[0]{\mathring{m}}
\newcommand{\wElm}[0]{\mathring{W}}
\newcommand{\uElm}[0]{\mathring{U}}
\newcommand{\bElm}[0]{\mathring{b}}

\newcommand{\f}[0]{f} 
\newcommand{\F}[0]{F}

\newcommand{\hide}[1]{}
\renewcommand{\show}[1]{#1}

\newcommand{\eps}[0]{\epsilon} 

\begin{abstract}
One of the most influential results in neural network theory is the universal approximation theorem \cite{hornik89, funahashi89, cybenko89} which states that continuous functions can be approximated to within arbitrary accuracy by single-hidden-layer feedforward neural networks. The purpose of this paper is to establish a result in this spirit for the approximation of general discrete-time linear dynamical systems---including time-varying systems---by recurrent neural networks (RNNs).  For the subclass of linear time-invariant (LTI) systems, we devise a quantitative version of this statement. Specifically, measuring the complexity of the considered class of LTI systems through metric entropy according to \cite{zamesDisc}, we show that RNNs can optimally learn---or identify in system-theory parlance---stable LTI systems. 
For LTI systems whose input-output relation is characterized through a difference equation, this means that RNNs can learn the
difference equation from input-output traces in a metric-entropy optimal manner.
\end{abstract}

\address{Chair for Mathematical Information Science, ETH Zurich\\
Sternwartstrasse 7, 8092 Zurich, Switzerland}
\maketitle

\section{Introduction}

During the past decade recurrent neural networks (RNNs) have revolutionized numerous machine learning applications, such as handwritten text recognition \cite{graves2009_handwriting}, speech recognition \cite{Graves2006_speach}, language translation \cite{sutskever2014seq2seq}, and modeling of complex game dynamics \cite{muZero}. 
Abstractly speaking, an RNN realizes a dynamical system mapping an input sequence to an output sequence through---in each time step---application of
a single-hidden-layer neural network to update 
a hidden state vector and compute the output signal sample.
It is hence natural to ask which classes of dynamical systems can be realized or approximated by RNNs.
This question is inspired by the well-known universal approximation theorem for feedforward neural networks \cite{hornik89, funahashi89, cybenko89}, which states that every continuous function on a compact interval can be approximated to within arbitrarily small error by a single-hidden-layer neural network, provided that the number of neurons is allowed to go to infinity as the approximation error approaches zero. 

The first central result in this paper establishes that RNNs universally exactly realize the class of linear dynamical systems, including time-varying systems. This universal linear dynamical system realization theorem builds on a strong representation theorem for general linear operators stemming from harmonic analysis \cite[Theorem~14.3.5]{Grochenig2001},\cite{Fefferman83,matz13}, which states that every ``reasonable'' linear operator can be written as a weighted superposition of time-frequency shift operators. Note that we conspicuously use the term ``realization theorem'' instead of ``approximation theorem'' as RNNs with real-valued weights will, indeed, be shown to exactly realize general linear dynamical systems.

The second central theme of this paper revolves around making the universal system realization theorem quantitative. Specifically, we consider classes of linear dynamical systems, quantify their complexity through metric entropy according to \cite{zamesConti, zamesDisc}, and ask how the number of bits needed to uniquely specify RNNs approximating systems in this class to within a prescribed error relates to the class's metric entropy.
This part of the theory we develop is restricted to linear time-invariant (LTI) systems for conceptual reasons. The main result we obtain states that RNNs with suitably quantized weights provide optimal coverings---in the sense of metric entropy---for the class of LTI systems with exponentially decaying impulse response. 
In control theory parlance, this says that RNNs can be trained to identify LTI systems with exponentially decaying impulse response in a metric-entropy optimal fashion. 
We also show that, equivalently, this means that certain classes of linear difference equations with constant coefficients can be learned optimally by RNNs.
The overall philosophy of the framework we propose is inspired by
the recently established Kolmogorov-Donoho rate-distortion theory \cite{Kolmogorov1959,Donoho2001sparse,Donoho96bitLevel,Donoho1998compression} for feedforward neural networks \cite{deep-approx-18,deepAT2019} which shows that deep neural networks provide optimal coverings for a wide range of function classes, such as unit balls in Besov spaces and in modulation spaces. 

We hasten to add that, throughout the paper, we are exclusively concerned with the fundamental representation capabilities of RNNs and do not consider the issue of learning algorithms, a topic that has been investigated in the context of LTI system identification in \cite{hardt2018, li2020courseOfMem}.

Previous work on the approximation of linear dynamical systems through neural networks deals with (linear and nonlinear) time-invariant systems and assumes that the system is specified in terms of a state space representation, concretely by a (time-invariant) next-state function 
which is approximated by a single-hidden-layer neural network, the existence of which is guaranteed by the classical universal approximation theorem
\cite{hornik89, funahashi89, cybenko89}. This approach leads, however, to the accumulation of errors over time so that most results along these lines are restricted to finite time horizons \cite{schaefer07, Sontag92neuralnets, funahashi93}. A notable exception in this regard is \cite{Matthews93}, which avoids error build-up by imposing an ``absolute summability'' condition on the system's possible state trajectories.
Nonetheless, all these results require that the system be characterized by a state space representation, the existence of which is not guaranteed for a given linear dynamical system \cite[Theorem 2.3.3]{Heij2021}. In the present paper, we do not impose such an existence assumption. In addition, our theory comprises time-varying systems and pertains to unbounded time horizons, but, as already mentioned, is restricted to linear systems.

As for our second central theme, namely metric-entropy-optimal RNN learning of LTI systems, to the best of our knowledge, such an approach has not been pursued before in the literature. Related previous work reported in 
\cite{li2020courseOfMem} quantifies the number of real-valued RNN weights required for a desired approximation quality, but does not attempt to
specify the approximating RNNs through bitstrings of finite length.

We furthermore want to highlight work on a non-recurrent neural network architecture, termed ``Deep operator network'' \cite{chenchen1995, Lu2021, lanthaler2021error}, which enables the universal approximation of nonlinear operators.
Finally, RNNs have also been investigated for the approximation of algorithms, with a prominent result \cite{sonntag95} proving that
RNNs with binary input and output sequences and rational weights can simulate any Turing machine.

\paragraph{Outline of the paper} In the remainder of this section, we provide preparatory material on RNNs and on harmonic analysis of general linear dynamical systems.
In Section \ref{sec:univ_ltv}, we develop the first central result of the paper, namely a universal realization theorem for discrete-time linear dynamical systems. 
In Section \ref{sec:metric_entropy_lti}, we introduce the concept of metric entropy of classes of LTI systems, based on which, in Section \ref{sec:quant_rnn}, we 
state the second central result establishing that RNNs realize LTI systems of exponentially decaying impulse response in a metric-entropy-optimal fashion. Appendices \ref{app:elmanrnn} and \ref{app:zproperties} summarize technical results needed in the main body of the paper.

\paragraph{Notation} 
Vectors are indexed starting with $\ell=1$. $\ind{\cdot}$ denotes the truth function which takes on the value $1$ if the statement inside $\{\cdot\}$ is true and equals $0$ otherwise. Sequences $x[t]\in\R$ are indexed by $t\in\mathbb{Z}$.
The $N\times N$ identity matrix is $\Imat{N}$ and $\Omat_{N}$ stands for the $N \times N$ all zeros matrix.
$\Ivec{N}$ and $\Ovec{N}$ denote the $N$-dimensional column vector with all entries equal to $1$ and $0$, respectively. $\log(\cdot)$ refers to the natural logarithm.
We write $f(\epsilon)\thicksim g(\epsilon)$ to mean $\lim_{\epsilon \rightarrow 0} \frac{f(\epsilon)}{g(\epsilon)} = 1$. Throughout the paper, constants are understood to be in $\R$ unless explicitly stated otherwise.

\subsection{Recurrent Neural Networks}

A recurrent neural network (RNN) is described by a hidden state vector sequence $h[t]$, the input signal $x[t]$, and the output signal $\rnnOut[t]$. In each time instant $t$, a single-hidden-layer neural network is applied to the concatenation of the input sample $x[t]$ and the previous state vector $h[t-1]$ to produce the current output sample $\rnnOut[t]$ and the new state vector $h[t]$. The formal definition is as follows.
\begin{definition}[Recurrent neural network]\label{def:rnn}
For $\hidLayerDim\in\N$ and hidden state dimension $\stateDim \in \N$, let $\Phi: \R^{\stateDim+1} \rightarrow \R^{\stateDim+1}$ be a feedforward neural network given by 
  \begin{equation}\label{eq:weights}
    \Phi(\vec{x}) = A_2\,\relu (A_1\vec{x} + b_1) + b_2, \quad x \in \R^{\stateDim+1},
  \end{equation}
 with weight matrices $A_1\in\R^{\hidLayerDim\times (\stateDim+1)}$, $A_2\in\R^{(\stateDim+1) \times \hidLayerDim}$, bias vectors $b_1\in \R^{\hidLayerDim}$, $b_2 \in \R^{\stateDim+1}$, and the ReLU activation function $\relu(x) = \max\{x,0\}, \, x\in\R$, applied element-wise.
 The recurrent neural network associated with $\Phi$ is the operator $\rnnOp{\Phi}: \seqSpace \rightarrow \seqSpaceC$ mapping
 input sequences $(x[t])_{t \geq 0}$ in $\R$ to output sequences $(\rnnOut[t])_{t\geq 0}$ in $\R$
  according to
  \begin{align}
  	\begin{pmatrix}
  		\rnnOut[t] \\
  		\h{t}
  		\end{pmatrix} = \Phi \left(
  		\begin{pmatrix}x[t]  \\ \h{t-1}  \end{pmatrix} \right)
  	, \; \forall t \geq 0,
  		\label{eq:rnn_seq}
  \end{align}
  where $\h{t}\in \R^\stateDim$ is the hidden state sequence with initial state $\h{-1} = 0_m$.
\end{definition}

\begin{remark}
    Classical RNN definitions 
    are often referred to as Elman networks \cite{elman90}, \cite[p274]{Goodfellow2016}. We show in Appendix \ref{app:elmanrnn} that our RNN definition does not afford increased generality over Elman networks as every RNN according to Definition \ref{def:rnn} can be converted into an Elman RNN. We decided, however, to work with the seemingly more general Definition \ref{def:rnn} for expositional simplicity.
\end{remark}
        We now introduce a decomposition of the weight matrix $A_2$ which will simplify the description of RNN constructions later in the paper. Specifically, we represent $A_2$ according to
        \begin{equation}\label{eq:a2_split}
            A_2 = 
            \begin{pNiceArray}{c}[margin]
                A_o A_r \\
                A_h
            \end{pNiceArray} \in \R^{(\stateDim+1) \times \hidLayerDim},
        \end{equation}
        where $\Ah \in \R^{\stateDim \times \hidLayerDim}$ is responsible for mapping to the next hidden state 
        and, for some $\repDim \in \N$, $\Ar \in \R^{\repDim \times \hidLayerDim}$ maps to an $\repDim$-dimensional virtual 
        representation which, in turn, is linearly combined through the weights $\Ao \in \R^{1 \times \repDim}$ to deliver the output $y[\cdot]$. 
        Consorting with this decomposition of $A_2$
        and noting that $b_2=0_{\stateDim+1}$ in all our concrete RNN constructions, 
        the hidden state sequence evolution can be written as
        \begin{align}
            \vec{h}[t] &= \Ah \vec{g}[t],  \label{eq:hidden_evolution_split}  \quad \forall t \geq 0,
        \end{align}
        where
        \begin{equation}\label{eq:intermediate_activation}
            \vec{g}[t] = \relu\left(A_1 \begin{pmatrix} x[t] \\ \vec{h}[t-1] \end{pmatrix} + b_1\right), \quad \forall t \geq 0,
        \end{equation}
        and the initialization is $\vec{h}[-1]=0_{\stateDim}$ as before.
        The output sequence is accordingly obtained as
        \begin{align}
            \vec{r}[t] &= \Ar \vec{g}[t], \label{eq:hidden_rep}\\
            \rnnOut[t] &= \Ao \vec{r}[t], \label{eq:simple_output}
        \end{align}
        where $\vec{r}[t]\in  \R^{\repDim}$ denotes the virtual representation sequence. We note that
        this virtual representation never actually manifests itself, it is introduced solely to simplify the specific RNN constructions later in the paper. Finally, we remark that, throughout, whenever we speak of ``weights'' of the RNN, this shall refer to nonzero entries both in the weight matrices $A_1,A_2$ and the bias vectors $b_1,b_2$.
\subsection{Harmonic Analysis of Linear Dynamical Systems}\label{sec:linear_system_intro}

We consider discrete-time causal linear systems $\sysOp$ mapping input sequences $x[\cdot] \in \seqSpace$ to output sequences $y[\cdot] \in \seqSpace$, and we use the convention 
\begin{equation}\label{eq:assum_x_right}
    x[t] = 0, \qquad \forall t < 0,
\end{equation}
which, by causality and linearity, implies 
$y[t]=0, \, \forall t<0$.

A fundamental result from harmonic analysis \cite[Theorem~14.3.5]{Grochenig2001}, in its incarnation for discrete-time systems, states 
that a wide class of linear operators, i.e., linear dynamical systems, can be represented as a weighted superposition of time-frequency shift operators according to
\begin{equation}\label{eq:delay_doppler}
    y[t] =  \sum_{\tau= 0}^{\infty} \int_0^1 \spread(\tau, \nu) x[t-\tau] e^{2\pi i \nu t} d\nu,
\end{equation}
with the weights given by the delay-Doppler spreading function $\spread(\tau, \nu)$. 
Alternatively, 
(\ref{eq:delay_doppler}) can be expressed in terms of the operator kernel, a.k.a. time-varying impulse response, $k[t, \tau]$, as
\begin{equation}\label{eq:ltv_imp_resp}
    y[t] = \sum_{\tau= 0}^{\infty} k[t, \tau] x[t-\tau],
\end{equation}
where $k[t, \tau]$ is related to the spreading function through an inverse Fourier transform according to
\begin{equation}
    k[t, \tau] =  \int_0^1 \spread(\tau, \nu) e^{2\pi i \nu t} d\nu.
\end{equation}
For a mathematically accessible introduction to this theory, we refer the interested reader to \cite{matz13}.

Throughout the paper, in an attempt to minimize the level of technical sophistication and expositional complexity, we will work with a fully discrete and finite-dimensional version of (\ref{eq:delay_doppler}) given by
\begin{equation}
    y[t] = \sum_{\tau=0}^{\maxTime-1} \sum_{f=0}^{F-1}\approxSpread(\tau, f) x[t-\tau] e^{2\pi i \frac{f}{F} t}, \label{eq:discret_LTV_intro}
\end{equation}
with $\maxTime,\, F \in \N$. In the continuous-time case the size of the spreading function support area
plays a critical role as there is a threshold beyond which the system becomes unidentifiable \cite{matz13}. While we will
not dwell on this matter, we simply note that in the setup considered here, the spread is given by $D \cdot F$, i.e., the total
number of time-frequency shifts the system induces.

\section{Universal Realization of Linear Dynamical Systems}\label{sec:univ_ltv}

In this section, we develop our first central result, a universal realization theorem for linear dynamical systems.
This will be effected by building on the spreading decomposition (\ref{eq:discret_LTV_intro}). Specifically, we first devise---in Lemma \ref{lem:rnn_lin_combi_of_history}---RNNs that realize time shifts, then---in Lemma \ref{lem:rnn_modulation}---RNNs implementing frequency shifts, and finally these building blocks are put together to obtain
an RNN that realizes a weighted superposition of time-frequency shifts
according to (\ref{eq:discret_LTV_intro}).

We start with RNNs that realize time shifts. For later reference, we actually construct more general RNNs that implement convolutions, i.e., weighted superpositions of time shifts.
\begin{lemma}[RNNs can realize time shifts and convolutions]\label{lem:rnn_lin_combi_of_history}
    Let $L \in \N$ and $\vec{k} \in \R^{L}$. There exists an RNN with input-output relation
    \begin{equation}\label{leq:output}
        \rnnOut[t] = \sum_{\ell=1}^{L} k_{\ell}\, x[t-(\ell-1)], \quad \forall t \geq 0,
    \end{equation}
    hidden state dimension $L-1$, and hidden state sequence satisfying
    \begin{equation}\label{leq:history}
            h_\ell[t] = x[t-(\ell-1)], \quad \forall t\geq 0, \ell \in \{1, \dots, L-1\}.
    \end{equation}
    
    \begin{proof}
     The proof is constructive in the sense of specifying the RNN as a function of the impulse response vector $k\,\in\,\R^{L}$. We start by choosing weight matrices and bias vectors such that (\ref{leq:history}) holds. 
     The basic idea is to design the network such that the past values of $x[\cdot]$ in the hidden state vector $\vec{h}$ are shifted downward by one position in each time step $t$, dropping the oldest value at the bottom of the vector and inserting the current value $x[t]$ at the top.
      To move the values through the non-linear activation function without modifying them, we employ the identity 
      \begin{equation}\label{teq:relu_identiy}
        x=\relu(x) - \relu(-x).
      \end{equation}
      We set 
      \begin{equation}\label{teq:timeshift_A1}
        A_1 = \begin{pNiceArray}{c}
        \Imat{L} \\
        -\Imat{L}
        \end{pNiceArray} \in \R^{2L \times L},
    \end{equation} 
    \begin{equation}\label{teq:timeshift_Ah}
        A_h = \begin{pNiceArray}{cccc}
        \Imat{L-1} & \Ovec{L-1} & 
        -\Imat{L-1} & \Ovec{L-1}
        \end{pNiceArray} \in \R^{(L-1) \times 2L},
    \end{equation} 
    $b_1=\Ovec{2L}$, and $b_2=\Ovec{L}$. 
    
    With these choices, the proof of (\ref{leq:history}) is now effected by induction over $t$. 
    First, we note that for $h[t]$ in (\ref{leq:history}) to constitute a valid hidden state sequence according to Definition \ref{def:rnn}, the initial state needs to satisfy $h[-1]=\Ovec{L-1}$. This follows directly from $x[t] = 0, \, \forall t<0,$ which is by assumption (\ref{eq:assum_x_right}), and also constitutes the base case of the induction argument.
    To establish the induction step, we assume that (\ref{leq:history}) holds for $t-1$ for some $t\geq0$, i.e.,
    \[ 
    h_\ell[t-1] = x[(t-1)-(\ell-1)] = x[t-\ell], 
    \]
    and show that---thanks to the choices for $A_1,A_h,b_1$, and $b_2$ made above---this implies validity of (\ref{leq:history}) for $t$.
    Using (\ref{teq:timeshift_A1}) and $b_1=0_{2L}$ in (\ref{eq:intermediate_activation}), one obtains
    \begin{align}
        \label{teq:timeshift_g}
        &\vec{g}[t] = \relu\left(A_1 \begin{pmatrix}x[t] \\ \vec{h}[t-1] \end{pmatrix} \right) 
        = \relu\left(A_1 \begin{pmatrix}x[t] \\ x[t-1] \\ \vdots \vspace*{1mm} \\ x[t-(L-1)] \end{pmatrix} \right) \\
        &= \begin{pmatrix} \relu(x[t]) & \hdots & \relu(x[t-(L-1)]) & \relu(-x[t]) & \hdots & \relu(-x[t-(L-1)])  \end{pmatrix}^T. \nonumber
    \end{align}
    Then, we evaluate (\ref{eq:hidden_evolution_split}) with $\Ah$ from (\ref{teq:timeshift_Ah}) and use (\ref{teq:relu_identiy}) to get
    \begin{align*}
        \vec{h}[t] = \Ah \vec{g}[t] = \begin{pmatrix} 
        \relu(x[t]) - \relu(-x[t]) \\
        \vdots \vspace*{1mm}\\
        \relu(x[t-(L-2)]) - \relu(-x[t-(L-2)])
        \end{pmatrix} = \begin{pmatrix} x[t] \\ \vdots \vspace*{1mm} \\ x[t-(L-2)] \end{pmatrix},
    \end{align*}
    or equivalently $h_\ell[t] = x[t-(\ell-1)], \forall \ell \in \{1,\dots,L-1\}$, 
    which establishes the induction step.

    It remains to realize the input-output relation (\ref{leq:output}). To this end, we set
    \begin{equation}
        \Ar  = \begin{pNiceArray}{cc}
        \Imat{L}  & 
        -\Imat{L}
        \end{pNiceArray} \in \R^{L \times 2L}, \qquad \Ao  = \vec{k}^T \in \R^ {1\times L},
    \end{equation}
    and use (\ref{teq:timeshift_g}), (\ref{eq:hidden_rep}), (\ref{eq:simple_output}), and (\ref{teq:relu_identiy}) to conclude that
    \begin{align*}
        \vec{r}[t] &= \Ar \vec{g}[t] =  \begin{pmatrix} 
        \relu(x[t]) - \relu(-x[t]) \\
        \vdots \vspace*{1mm}\\
        \relu(x[t-(L-1)]) - \relu(-x[t-(L-1)]) \\
        \end{pmatrix} = \begin{pmatrix} x[t] \\ \vdots \vspace*{1mm} \\ x[t-(L-1)] \end{pmatrix}, \\
        \rnnOut[t] &= \vec{k}^T \vec{r}[t] = \sum_{\ell=1}^{L} k_\ell \, x[t-(\ell-1)],
    \end{align*}
    which, in turn, completes the proof.
    \end{proof}
\end{lemma}

\begin{remark}
An RNN realizing a time shift by $m$ instants, as needed in (\ref{eq:discret_LTV_intro}),
is now obtained from Lemma \ref{lem:rnn_lin_combi_of_history} by choosing the impulse response vector $k\,\in\,\R^{L}$ such that it has a $1$ in the $(m+1)$-th entry and zeros elsewhere.
\end{remark}

The next step in our program is to construct an RNN that realizes frequency shifts by integer multiples of $1/F$, again as needed in (\ref{eq:discret_LTV_intro}).
As this operation corresponds to multiplication of the input signal by a complex exponential, it produces complex outputs $y[\cdot]$. In slight abuse of Definition {\ref{def:rnn}}, where all weight matrices and bias vectors are real-valued, for ease of exposition, we will here allow complex weights in the output layer, specifically for the quantities $\Ao$ and $\Ar$ in ({\ref{eq:a2_split}}). As $\Ao$ and $\Ar$ do not appear in the state evolution equations ({\ref{eq:hidden_evolution_split}}) and ({\ref{eq:intermediate_activation}}), it is guaranteed that the activation function $\rho$ continues to be applied to real-valued quantities only. 

\begin{lemma}[RNNs can realize frequency shifts]\label{lem:rnn_modulation}
    Let $F\in\N$ and $f\in \{0, \dots, F-1\}$. There exists an RNN that realizes the input-output mapping
    \begin{equation}
        \rnnOut[t] = x[t]\, e^{2\pi i \frac{f}{F}t}, \quad \forall t\geq 0.  \label{eq:output-modulation}
    \end{equation}
    
    \begin{proof}
    Again the proof is constructive in the sense of specifying the RNN. Throughout the proof, unless explicitly stated otherwise, relations involving $t$ apply for all $t \geq 0$.
    We start by noting that the function $e^{2\pi i \frac{f}{F}t}$ is $F$-periodic in $t\in \N$. 
    This $F$-periodicity motivates the choice of an $(F-1)$-dimensional hidden state sequence $\h{t}\in \{0, 1\}^{(F-1)}$ encoding the current position within the fundamental period. Specifically, our construction will be seen to ensure
    \begin{equation}\label{eq:period_state_cond}
        \quad h_\ell[t] = \ind{ ((t+1) \bmod F)\, = \, \ell }, \quad \forall \ell \in \{1, \dots, F-1\}.
    \end{equation}
    The hidden state vector at time $t$ hence contains a one at 
    position $(t+1) \bmod F$ or equals the all-zeros vector at the end of each period, i.e., when $(t+1) \bmod F=0$. 
    We will realize (\ref{eq:period_state_cond}) by appropriate choice of the RNN weight matrices $A_1,A_2$ and bias vectors $b_1,b_2$ and the proof will proceed by induction. First, we note that for $h[t]$ in (\ref{eq:period_state_cond}) to constitute a valid hidden state sequence according to Definition \ref{def:rnn}, the initial state needs to satisfy $\h{-1}=\Ovec{F-1}$, which at the same time would constitute the base case $t=-1$ of the induction argument. The relation $\h{-1}=\Ovec{F-1}$ now follows
    independently of the choices for $A_1,A_2,b_1,b_2$ and is simply by virtue of the index $0$ not being contained in the set $\{1,\dots,F-1\}$ so that the truth function on the RHS of (\ref{eq:period_state_cond}) yields the all-zeros vector.
    For the induction step, we assume that (\ref{eq:period_state_cond}) holds for $t-1$ for some $t\geq0$, i.e., 
    $h_\ell[t-1] = \ind{ (t \bmod F)\, = \, \ell }, \, \forall \ell \in \{1, \dots, F-1\}$.
    Next, we set
    \begin{equation}\label{eq:period_weights_def}
        \periodMatrix{} := \begin{pmatrix}
        -1 & -1 & \dots & -1 \\
        1 & 0 & \dots & 0 \\
        0 & 1 & \dots & 0 \\
        \vdots &  & \ddots &\vdots \\[1mm]
        0 & 0 & \dots & 1
        \end{pmatrix} \in \R^{F \times (F-1)},
        \quad
        \periodBias := \begin{pmatrix}1 \\ 0 \\0 \\ \vdots \\[1mm] 0 \end{pmatrix} \in \R^F,
    \end{equation} 
    and define the sequence $\vec{e}[t] \in \{0, 1\}^F$ 
    according to
    \begin{equation}\label{teq:def_period_state_e}
        \vec{e}[t] := \periodMatrix \vec{h}[t-1] + b_e.
    \end{equation}
    Direct calculation now yields
    \begin{equation}\label{eq:period_state_e_characterisation}
        \qquad e_\ell[t] = \ind{(t \bmod F)+1 \, = \, \ell}, \quad \forall \ell \in \{1, \dots, F\}.
    \end{equation}
    That is, $\vec{e}[t]$ indicates its argument $t$, modulo $F$ to account for $F$-periodicity of $e^{2\pi i \frac{f}{F}t}$, 
    through a one at the corresponding position in the period and, unlike the state vector, never equals the all-zeros vector.
    We now use $\vec{e}[t]$ to construct an indicator function applied to the input signal.
    To this end, we first recall that RNNs according to Definition \ref{def:rnn} accept input signals in $\ell_{\infty}$, and set $C=\|x\|_{\ell_{\infty}}$.
    Next, for all $t \geq 0$, consider the sequence
    \begin{align}
    \begin{split}
        \vec{\xPeriod}[t] &= \relu \left (
        \begin{pNiceArray}{cc}
            \Ivec{F} & 2C \periodMatrix
        \end{pNiceArray}
            \begin{pmatrix}
                x[t] \\ \vec{h}[t-1]
            \end{pmatrix}
            + 2C  \periodBias - C \Ivec{F}
        \right ) 
        \\
        &=\relu \left (\Ivec{F}\, x[t] + 2C \vec{e}[t] - C \Ivec{F}
        \right),
        \end{split} \label{teq:period_x}
    \end{align}
    where we used (\ref{teq:def_period_state_e}).
    Equivalently, we can express (\ref{teq:period_x}) as
    \begin{equation}\label{teq:period_x_elements}
        \xPeriod_\ell[t] = \relu(x[t] +2C\ind{ ( t \bmod F ) + 1 \, = \, \ell} - C) = (x[t] + C) \ind{ ( t \bmod F ) + 1 \, = \, \ell},
    \end{equation}
    for $\ell \in \{1, \dots, F\}$, where we made use of $|x[t]| \leq C$.
    We proceed to set 
    \begin{equation}\label{eq:modulation_weights_layer_one}
        A_1 = \begin{pNiceArray}{cc}
            \Ivec{F} & 2C  \periodMatrix \\
            \Ovec{F} & \periodMatrix
        \end{pNiceArray} \in \R^{(2F) \times F},
        \qquad
        b_1 = \begin{pmatrix}
            2C \periodBias -C 1_{F} \\ \periodBias
        \end{pmatrix} \in \R^{2F},
    \end{equation}
    and $b_2=0_{F}$. Inserting into (\ref{eq:intermediate_activation}) yields
    \[
        \begin{split}\label{teq:modulation_g}
        \vec{g}[t] &= \relu\left(  
        A_1
        \begin{pmatrix}
            x[t] \\ \vec{h}[t-1]
        \end{pmatrix} + 
        b_1
        \right )
        =
        \begin{pNiceArray}{c}
            \vec{\xPeriod}[t]\\
            \vec{e}[t]
        \end{pNiceArray},
        \end{split}
    \]
    where we employed (\ref{teq:period_x}) and (\ref{teq:def_period_state_e}), and used the fact that
    $\relu{(\vec{\periodState}[t])} = \vec{\periodState}[t]$ as the entries of $\vec{\periodState}[t]$ equal either $0$ or $1$. 
    Next, we let
    \begin{align}
        \begin{split}\label{teq:last_layer_weights_modulation}
        \Ah & = \begin{pmatrix}
            \Omat_{F-1} &
            \Ovec{F-1} &
            \Imat{F-1} & 
            \Ovec{F-1}
        \end{pmatrix},
        \\
        \Ar & = \begin{pmatrix}
            \Imat{F} & -C\, \Imat{F}
        \end{pmatrix},
        \\ 
        \Ao & = \begin{pmatrix}
            e^{2\pi i \frac{0}{F} f} &
            e^{2\pi i \frac{1}{F} f} &
            \hdots &
            e^{2\pi i \frac{F-1}{F} f} 
        \end{pmatrix}.
        \end{split}
    \end{align}
    We are now ready to finalize the induction step.
    From $\vec{h}[t] = \Ah \vec{g}[t]$ we get
    \begin{equation}
        h_\ell[t] = e_\ell[t] = \ind{(t \bmod F) +1 \, = \, \ell} = \ind{((t+1) \bmod F) \, = \, \ell}, \, \forall \ell \in \{1, \dots, F-1\}, \label{eq:hleleq}
    \end{equation}
    where we used that 
    \begin{equation}
    (t \bmod F) + 1 = ((t+1) \bmod F), \label{eq:lmodF}
    \end{equation}
    for all $t$ with $(t \bmod F) \neq F-1$. For $t$ such that
    $(t \bmod F) = F-1$, the LHS of (\ref{eq:lmodF}) equals $F$ while the RHS is equal to $0$; as the indices $0$ and $F$ do not occur in the set $\{1, \dots, F-1\}$, we trivially have equality between the last two expressions in (\ref{eq:hleleq}). 
    This establishes (\ref{eq:period_state_cond}) and thereby completes the induction step. 
    
    It remains to prove that the input-output relation of the RNN specified along the way is, indeed, given by (\ref{eq:output-modulation}).
    Using (\ref{teq:last_layer_weights_modulation}), (\ref{teq:period_x_elements}), and (\ref{eq:period_state_e_characterisation}) in $\vec{r}[t] = \Ar \vec{g}[t]$, it follows that
    \begin{align}
    \begin{split}
        r_\ell[t] &= \xPeriod_\ell[t] - C e_\ell[t] \\
        &= (x[t] + C) \ind{ ( t \bmod F ) + 1\, = \, \ell} - C \ind{ ( t \bmod F ) + 1\, = \, \ell} \\
        &= x[t] \ind{ ( t \bmod F ) + 1\, = \, \ell}, \quad \forall \ell \in \{1, \dots, F\}.
    \end{split}
    \end{align}
    The output signal is hence given by 
    \begin{align}
        \begin{split}
        \rnnOut[t] &= \Ao \vec{r}[t] \\
        &= \sum_{\ell=1}^{F}   e^{2\pi i \frac{\ell-1}{F} f}  x[t]\, \ind{ ( t \bmod F ) + 1 = \ell}\\
        &= x[t]\, e^{2\pi i \frac{( t \bmod F )}{F} f} \\
        &= x[t]\, e^{2 \pi i \frac{t}{F} f},
        \end{split}
    \end{align}
    where, in the last step, we made use of the $F$-periodicity of $e^{2 \pi i  \frac{t}{F} f}$.
    This completes the proof.
    \end{proof}
\end{lemma}

Having established the RNN realizations of the basic building blocks of the spreading representation (\ref{eq:discret_LTV_intro}), 
namely RNNs that realize time shifts (or, more generally, convolutions) and frequency shifts, we proceed to devise RNNs that implement weighted linear combinations of time-frequency shift operators.
This entails showing that linear combinations of compositions of time shift RNNs and frequency shift RNNs are again RNNs.
As opposed to feedforward networks where compositions and linear combinations trivially preserve the feedforward structure \cite{deepAT2019}, this is not obvious in the RNN case. The basic idea underlying the construction provided next is hidden-state sharing across component networks, which not only preserves the RNN structure, but also leads to an economical---in terms of the number of nonzero weights---RNN realization.

\begin{lemma}[RNNs can realize LTV systems]\label{lem:rnn_ltv_construction}
    Let $\maxTime, F \in \N$ and consider the spreading function $\approxSpread(\tau, f) \in \C$, $\tau \in \{0, \dots, \maxTime-1\}$, $f \in \{0, \dots, F-1\}$. There exists an RNN that realizes the input-output relation 
    \begin{equation}
         \rnnOut[t] = \sum_{\tau=0}^{\maxTime-1} \sum_{f=0}^{F-1}\approxSpread(\tau, f) x[t-\tau] e^{2\pi i \frac{f}{F} t}, \quad \forall t \geq 0. \label{eq:discrete_LTV}
    \end{equation}
    \begin{proof}

        There are two main components in the construction of the RNN realizing the desired input-output relation, namely the composition of time shift and frequency shift operators and weighted linear combinations thereof. The latter is easily realized through proper choice of the output layer weight matrix $A_2$, whereas the former requires more effort. Specifically, we will design the RNN such that its hidden state vector concatenates the hidden state vectors of the time shift and the frequency shift RNNs in Lemmas \ref{lem:rnn_lin_combi_of_history} and \ref{lem:rnn_modulation}, respectively, and that this concatenated hidden state vector follows the hidden state evolution equations of the constituent time shift and frequency shift networks. Concretely, our goal will be to design the RNN such that its hidden state vector\footnote{Note that the symbols $\hHist$ and $\hPS$ do not refer to derivatives of $\vec{h}$ in any form.}
        is given by
    \begin{equation}
    \vec{h}[t] = \begin{pmatrix} \vec{\hHist}[t] \\[1mm] \vec{\hPS}[t] \end{pmatrix}, \label{eq:super-hidden-state}
    \end{equation}
    where $\vec\hHist\in \R^{(\maxTime-1)}$ corresponds to the hidden state of the convolution RNN from Lemma \ref{lem:rnn_lin_combi_of_history} particularized for pure time shifts, and $\vec\hPS \in \{0, 1\}^{(F-1)}$ represents the hidden state of the frequency shift RNN in Lemma \ref{lem:rnn_modulation}. The component vector sequences $\vec{\hHist}[t]$ and $\vec{\hPS}[t]$ now need to
    follow the state evolution laws in (\ref{leq:history}) and (\ref{eq:period_state_cond}), respectively, i.e.,
    \begin{align}
        \qquad \hHist_\ell [t ] = x[t-(\ell-1)]&, \quad \forall \ell \in \{1, \dots, \maxTime-1\} \label{teq:h_part_history}\\
        \qquad \hPS_\ell [t ] = \ind{((t+1) \bmod F) \, = \, \ell}&, \quad \forall \ell \in \{1, \dots, F-1\},\label{teq:h_part_period_state}
    \end{align}
    both for all $t\geq 0$. The approach we follow will be as in the proofs of Lemmas \ref{lem:rnn_lin_combi_of_history} and \ref{lem:rnn_modulation}, namely, we proceed by induction and in the process specify the network weight matrices and bias vectors to make the induction work out. 
    The proof will be finalized by showing 
    how the state vector $h[t]$ following (\ref{teq:h_part_history}) and (\ref{teq:h_part_period_state}) leads to the desired overall input-output relation by proper choice of $A_2$.
    
    The base case $t=-1$ of the induction, i.e., $h[-1]=0_{\maxTime+F-2}$, follows as the base case in Lemma \ref{lem:rnn_lin_combi_of_history} is by virtue of 
    $x[t]=0, \forall t < 0$, and that in Lemma \ref{lem:rnn_modulation} holds as a consequence of the definition of the state vector. Notably, for both components, $\hHist_\ell [t]$ and $\hPS_\ell [t ]$, the base case follows independently of the choices of the weight matrices and bias vectors.
    We remark that the base case also establishes that the initial state $h[-1]$ of the hidden state sequence in (\ref{eq:super-hidden-state})---by virtue of being equal to the all zeros vector---conforms with Definition \ref{def:rnn}.
    
    To establish the induction step, we will have to choose $A_1,A_2,b_1,$ and $b_2$ appropriately. Concretely, we start by assuming that (\ref{teq:h_part_history}) and (\ref{teq:h_part_period_state}) hold for $t-1$ for some $t \geq 0$, and set
        \begin{align}
\NiceMatrixOptions
{nullify-dots,code-for-first-col = \color{blue},code-for-first-row=\color{blue}, code-for-last-col=\color{blue} }
        A_1 = \begin{pNiceArray}{cccc|c}[first-row, first-col]
        & \Hdotsfor[line-style={solid,<->},shorten=0pt]{4}^{\maxTime} & \Hdotsfor[line-style={solid,<->},shorten=0pt]{1}^{F-1} \\ 
         \Vdotsfor[line-style={solid,<->},shorten=2pt]{4}_{\maxTime F }
         & \Ivec{F} & \Ovec{F} & \dots & \Ovec{F}    &    2C \periodMatrix \\
         & \Ovec{F} & \Ivec{F} & \dots & \Ovec{F}    &    2C \periodMatrix  \\
         & \vdots & \vdots & \ddots & \vdots         &    \vdots  \\[1mm]
         & \Ovec{F} & \Ovec{F} & \dots & \Ivec{F}    &    2C \periodMatrix\\
         \hline 
         \Vdotsfor[line-style={solid,<->},shorten=2pt]{2}_{\hspace{-0.002cm} 2\maxTime}
         &  \Block{1-4}{   \Imat{\maxTime} } &&&   & \Block{2-1}{\Omat} \\
         & \Block{1-4}{   -\Imat{\maxTime} } &&&    & \\
         \hline
         \Vdotsfor[line-style={solid,<->},shorten=2pt]{1}_{F}
         & \Block{1-4}{\Omat} &&&                                        & \periodMatrix 
         \end{pNiceArray},
         \qquad
         b_1 = 
         \begin{pNiceArray}{c}
         2C\periodBias - C\Ivec{F} \\2C \periodBias -C\Ivec{F} \\ \vdots \\[1mm] 2C\periodBias - C\Ivec{F} \\ \hline 
         \Block{2-1}{\Omat} \\ \\ \hline \periodBias
         \end{pNiceArray} ,
    \end{align}
    where $\periodMatrix, \periodBias$ are as defined in (\ref{eq:period_weights_def}),
    $C=\|x\|_{\ell_{\infty}}$, and the unsubscripted $\Omat$ symbols stand for all zeros matrices of appropriate dimensions. The bias vector $b_2$ is chosen as $b_2=0_{D+F-1}$. The 
    first $\maxTime$ columns of $A_1$ operate on 
    \begin{equation}\label{teq:ltv_rnn_history}
    \begin{pmatrix} x[t] \\ \vec{\hHist}[t-1] \end{pmatrix}= \begin{pmatrix} x[t] \\ \vdots \\[1mm] x[t-(\maxTime-1)] \end{pmatrix}
    \end{equation}
    and the last $F-1$ columns multiply $\vec{\hPS}[t-1]$. 
    Further, $A_1$ is divided vertically into three parts. The first $\maxTime F$ rows produce, for each time shift (including the shift by $0$ time instants), a representation akin to (\ref{teq:period_x}), the middle $2\maxTime$ rows correspond to (\ref{teq:timeshift_A1}) in the time shift RNN construction, and
    the last $F$ rows pertain to the frequency shift RNN, specifically to (\ref{teq:def_period_state_e}).
    It is hence natural to think of $\vec{g}[t]$ from (\ref{eq:intermediate_activation}) in three parts according to
    \begin{equation}\label{teq:g_spread}
        \vec{g}[t] = \relu\left( A_1 \begin{pmatrix} x[t] \\ \vec{\hHist}[t-1] \\ \vec{\hPS}[t-1] \end{pmatrix} + b_1 \right) = \;
        \NiceMatrixOptions
{nullify-dots,code-for-first-col = \color{blue},code-for-first-row=\color{blue}, code-for-last-col=\color{blue} }
        \begin{pNiceArray}{c}[first-col]
            \Vdotsfor[line-style={solid,<->},shorten=2pt]{4}_{\hspace{-0.05cm} { \maxTime F}} &
            \vec{\xPeriod}[t,0] \\
            &  \vec{\xPeriod}[t,1]\\
            &  \vdots \\
            &  \vec{\xPeriod}[t,\maxTime-1]\\
            \hline
            \Vdotsfor[line-style={solid,<->},shorten=2pt]{1}_{\hspace{-0.05cm} { 2 \maxTime}}
            & \vec{\gHist}[t] \\
            \hline
            \Vdotsfor[line-style={solid,<->},shorten=2pt]{1}_{\hspace{-0.05cm} { F}}
            & \vec{\periodState}[t] \\
        \end{pNiceArray},
    \end{equation}
    where, following the steps leading to (\ref{teq:period_x_elements}) in Lemma \ref{lem:rnn_modulation}, the vectors $\vec{\xPeriod}[t, \tau]$, $\tau \in \{0, \dots, \maxTime-1 \}$, are obtained as
    \begin{equation}\label{teq:spread_x_period}
        \xPeriod_\ell[t, \tau] =(x[t-\tau] + C) \ind{ ( t \bmod F ) + 1 = \ell},
    \end{equation}
    $\vec{\gHist}[t]$ equals $g[t]$ in (\ref{teq:timeshift_g}) with $L=D$, and $\vec{\periodState}[t]$ is as in (\ref{eq:period_state_e_characterisation}).

    We proceed to specify $\Ah$, the submatrix of $A_2$, which maps to the next hidden state according to (\ref{eq:hidden_evolution_split}), as
        \begin{align}\label{eq:second_layer_weighs}
        \Ah = \; 
    \NiceMatrixOptions
{nullify-dots,code-for-first-col = \color{blue},code-for-first-row=\color{blue}, code-for-last-col=\color{blue} }
        \begin{pNiceArray}{c|cccc|cc}[first-col, first-row]
            &  \Hdotsfor[line-style={solid,<->},shorten=0pt]{1}^{\maxTime F} &  \Hdotsfor[line-style={solid,<->},shorten=0pt]{4}^{2\maxTime} & 
             \Hdotsfor[line-style={solid,<->},shorten=0pt]{2}^{F} 
             \\ 
            \Vdotsfor[line-style={solid,<->},shorten=2pt]{1}_{\hspace{-0.05cm} {\scriptscriptstyle \maxTime-1}} &  \Omat             & \Imat{\maxTime-1} & \Ovec{\maxTime-1} &  -\Imat{\maxTime-1} & \Ovec{\maxTime-1} &\Block{1-2}{\Omat} & 
            \\
            \hline 
            \Vdotsfor[line-style={solid,<->},shorten=2pt]{1}_{\hspace{0.2cm} {\scriptscriptstyle F-1}} &  \Omat             & \Block{1-4}{\Omat} &&&& \Imat{F-1} & \Ovec{F-1}
        \end{pNiceArray},
    \end{align}
    where again the unsubscripted $\Omat$ symbols refer to all-zeros matrices of appropriate dimensions. Carrying out the state transition for the concatenated hidden state vector (\ref{eq:super-hidden-state}) according to (\ref{eq:hidden_evolution_split}) with $A_h$ in (\ref{eq:second_layer_weighs}) and $g[t]$ in (\ref{teq:g_spread}) yields (\ref{teq:h_part_history}) directly and (\ref{teq:h_part_period_state}) upon using the last identity in (\ref{eq:hleleq}).

    Next, with the $F \times F$ (unnormalized) DFT matrix 
    $[\AF]_{f,n}=e^{2\pi i \frac{f}{F}n},f\in\{0,\dots,F-1\},\, n \in \{0,\dots,F-1\}$, 
    we define
        the vectors 
    \begin{align}\label{teq:modulatedXdef}
        \vec{\modulatedX}[t, \tau] := 
        \begin{pmatrix} \AF & -C\AF \end{pmatrix}
        \begin{pmatrix} \vec{\xPeriod}[t, \tau] \\ \vec{\periodState}[t] \end{pmatrix}, \quad \tau \in \{0,\dots,\maxTime-1\},
    \end{align}
    and note, by direct calculation, that
    \begin{align}\label{teq:modulated_f}
        \modulatedX_f[t, \tau] = x[t-\tau] e^{2\pi i \frac{f-1}{F} t}, \quad f \in \{ 1,\dots,F \}.
    \end{align}
    Now we stack the frequency-shifted versions of $x[t-\tau]$ in (\ref{teq:modulated_f}) in the virtual representation sequence $r[t]$ by setting
    \begin{align}
            \Ar & = \;
                    \NiceMatrixOptions
{nullify-dots,code-for-first-col = \color{blue},code-for-first-row=\color{blue}, code-for-last-col=\color{blue} }
            \begin{pNiceArray}{cccc | c | c}[first-col, first-row]
            & \Hdotsfor[line-style={solid,<->},shorten=0pt]{4}^{\maxTime F} &       \Hdotsfor[line-style={solid,<->},shorten=0pt]{1}^{2\maxTime} & 
             \Hdotsfor[line-style={solid,<->},shorten=0pt]{1}^{F} 
             \\
              \Vdotsfor[line-style={solid,<->},shorten=2pt]{4}_{\hspace{-0.05cm} {\scriptscriptstyle \maxTime F}}
            & \AF & \Omat & \dots & \Omat & \Omat & -C \AF 
            \\ 
            & \Omat & \AF & \dots & \Omat & \Omat & -C \AF 
            \\ 
            & \vdots & \vdots  & \ddots &  \vdots & \vdots & \vdots  
            \\[1mm]
            & \Omat & \Omat & \dots & \AF & \Omat & -C \AF \\ 
        \end{pNiceArray}, \label{eq:ar-matrix-ltv-rnn}
    \end{align}
    which yields 
    \begin{align}
        \vec{r}[t] = \Ar \vec{g}[t] = \begin{pmatrix}
            \vec\modulatedX[t, 0] \\
            \vec\modulatedX[t, 1] \\
            \vdots \\[1mm]
            \vec\modulatedX[t, \maxTime-1] \\
        \end{pmatrix}. \label{eq:ar-matrix-multiplication}
    \end{align}
    This finalizes the first part of the construction, namely the composition of time shifts and frequency shifts according to (\ref{teq:modulated_f}). 
    We are left with the weighted superposition of the time-frequency-shifted versions of $x[t]$ implementing the input-output relation (\ref{eq:discrete_LTV}).
    To this end, we set, for $\tau \in \{0, \dots, \maxTime-1 \}$,
    \begin{equation}\label{teq:def_spread_vec}
        \vec{\approxSpread}(\tau, \cdot) := \begin{pmatrix}
        \approxSpread(\tau, 0) \\
        \approxSpread(\tau, 1) \\
        \vdots \\[1mm]
        \approxSpread(\tau, F-1) 
        \end{pmatrix} \in \C^{F}
    \end{equation}
    and take
    \begin{equation}
        \Ao = \begin{pmatrix}
            \vec\approxSpread(0, \cdot)^T &
            \vec\approxSpread(1, \cdot)^T &
            \hdots &
            \vec\approxSpread(\maxTime-1, \cdot)^T 
        \end{pmatrix},
    \end{equation}
    which results in
    \begin{align}
       \rnnOut[t] &=  \Ao \vec{r}[t] = \sum_{\tau=0}^{\maxTime-1} \vec\approxSpread(\tau, \cdot)^T \vec\modulatedX[t, \tau] \\
        &= 
        \sum_{\tau=0}^{\maxTime-1} \sum_{f=0}^{F-1} \approxSpread(\tau, f) x[t-\tau] e^{2\pi i \frac{f}{F} t},
    \end{align}
    as desired.
    \end{proof}
\end{lemma}

We note that the RNN constructed in Lemma \ref{lem:rnn_ltv_construction} has $\mathcal{O}(\maxTime F)$ non-zero weights, i.e., the ``size'' of the network is proportional to the spread of the system it is to realize. 
This insight is based on the fact that the virtual representation sequence $r[t]$ is never actually manifested. Hence, by (\ref{eq:a2_split}) only the product $\Ao\Ar$ has to be stored instead of the (bigger) individual matrices $\Ao$ and $\Ar$. Finally, we remark that the magnitudes of the RNN weights in the proof of Lemma \ref{lem:rnn_ltv_construction} depend on the $\ell_{\infty}$-norm of the inputs the RNN accepts.

\section{Metric Entropy of LTI Systems} \label{sec:metric_entropy_lti}

Having established that RNNs can universally realize linear dynamical systems with network size proportional to the spread of the system, we proceed to develop a deepened and more quantitative theory along those lines. Specifically, we shall be interested in the approximation of classes of linear dynamical systems to within a prescribed worst-case (within the class) error $\epsilon$ through RNNs that can be specified by bitstrings of finite length. Of particular interest will be the scaling behavior of the required length of the bitstring as a function of $\epsilon$ and, in particular, whether RNNs can achieve the fundamental limit---over all possible system approximation methods---on this scaling behavior. Answering this question requires the concept of metric entropy of linear systems, a topic originating from control theory \cite{zamesConti, zamesDisc}. The aim of the present section is to
introduce this concept, with the presentation geared towards our purposes.
We restrict ourselves to LTI systems for conceptual reasons.

A linear dynamical system is time-invariant if the operator kernel 
$k[t,\tau]$ in (\ref{eq:ltv_imp_resp}) is a function of $\tau$ only, i.e., the input-output relation of the system is given by the 
convolution of the input signal $x[\cdot]$ with the impulse response $k[\cdot]$ according to
\begin{equation}\label{eq:sys_is_conv}
    (\sysOp x) [t] = \sum_{\tau=0}^{\infty} \impResp[\tau]x[t-\tau] =: (\impResp *x)[t],
\end{equation}
where, as before, we assume that $x[t]=0,\,k[t]=0$, for $t < 0$, that is we consider one-sided input signals and causal systems. 
We shall frequently make use of the one-sided $\mathcal{Z}$-transform for $\ell^2$-signals defined as\footnote{
Note the positive exponents of $z$ in the definition. This convention is chosen to maintain consistency with Definition \ref{def:hardy} below adopted from \cite{zamesDisc}. 
} 
\begin{equation}\label{eq:z_def}
    (\Z {x[\cdot]})(z) = \sum_{t=0}^{\infty} x[t]z^t,  \quad |z| < 1.
\end{equation}
Whenever there is no source of ambiguity, we shall use capital letters to denote the $\mathcal{Z}$-transform according to $X(z) = (\Z{x[\cdot]})(z)$. 
Next, we note the well-known relation
\begin{equation}\label{eq:conv_is_multiplication}
    (\Z{(\mathcal{L} x)[\cdot]})(z)=(\Z{(k*x)[\cdot]})(z) = K(z) \cdot X(z),
\end{equation}
where $K(z) := (\Z{\impResp[\cdot]})(z)$ is commonly referred to as the system's transfer function.  

We proceed to establish the concept of metric entropy 
of classes of LTI systems largely following \cite{zamesConti, zamesDisc}. On a conceptual level, this complexity notion allows to formulate answers to the following question:
\textit{Given a class of LTI systems, how many bits of information do we need to identify a specific system in the class to within a prescribed error?}
To formalize matters, we start by defining the metric entropy of general sets.

\begin{definition}[\cite{Wainwright2019}] \label{def:metric_entropy}
  Let $(\mathcal{X}, \rho)$ be  a metric space. An $\epsilon$-covering of a compact set $\mathcal{C} \subseteq \mathcal{X}$ with respect to the metric $\rho$ is a set of points $\{ x_1, \dots, x_N \} \subset \mathcal{C}$ such that for each $x \in \mathcal{C}$, there exists an $i \in [1, N]$ so that $\rho(x, x_i) \leq \epsilon$. The $\epsilon$-covering number $N(\epsilon; \mathcal{C}, \rho)$ is the cardinality of a smallest $\epsilon$-covering of $\mathcal{C}$ and $\mEnt(\epsilon; \mathcal{C}, \rho) := \log_2( N(\epsilon; \mathcal{C}, \rho))$ is the metric entropy of $\mathcal{C}$.
\end{definition}

As LTI systems are uniquely determined by their impulse response, we shall incarnate the concept of ``classes of LTI systems" by considering compact sets of impulse responses. More specifically, motivated by \cite{zamesDisc}, we consider systems with exponentially decaying impulse response, that is, the 
set of LTI systems 
characterized by
\begin{equation}\label{eq:decay_imp_resp}
    \mathcal{C}(\mC{}, \eC{}) := \{ \sysOp \mid |k_{\sysOp}[t]| \leq \mC{}e^{-\eC{}t},\, \forall t\geq 0 \}, \quad \mC{},\eC{} > 0,\end{equation}
where $\impResp_{\sysOp}[\cdot]$ denotes the impulse response of the system $\sysOp$. 
The constants $\eC{}$ and $\mC{}$ quantify the decay behavior of the system memory. Note that the set $\mathcal{C}(\mC{}, \eC{})$ encompasses 
exponentially decaying impulse responses of arbitrary decay rate according to
\begin{equation}
\mathcal{C}(\mC{}, \log(1/\expBaseC{})) = \{ \sysOp \mid |k_{\sysOp}[t]| \leq \mC{}\expBaseC{}^{t},\, \forall t\geq 0 \}, \quad \mC{}>0, \, \expBaseC{} \in (0,1). \label{eq:exp-decay-general-exponent}
\end{equation}

Next, we equip the ambient space $\mathcal{X}$ with a suitable metric which quantifies the distance between LTI systems, or equivalently their impulse responses. To this end, we first define
Hardy spaces and norms of transfer functions as follows. 
\begin{definition}[{\cite[Chapter~17]{Rudin1987}}] \label{def:hardy} 
For the transfer function $K(z)$, we define the Hardy norms    \begin{align}
        \norm{\transF}_{\Hardy{2}} &:= \sqrt { \sup_{r<1} \frac{1}{2\pi}\int_0^{2\pi} |\transF(re^{i\theta})|^2 d\theta, }\\
        \norm{\transF}_{\Hardy{\infty}} &:= \sup_{|z|<1} |\transF(z)|.
    \end{align}
The corresponding Hardy spaces are given by $\Hardy{2}=\{K(\cdot)\,|\, \norm{\transF}_{\Hardy{2}} < \infty\}$ and $\Hardy{\infty}=\{K(\cdot)\,|\, \norm{\transF}_{\Hardy{\infty}} < \infty\}$.
\end{definition}
The distance between the LTI systems $\sysOp$ and $\sysOp'$ with transfer functions $\transF(z)$ and $\transF'(z)$, respectively, both in $\Hardy{\infty}$, is now defined as
\begin{equation}\label{eq:metric}
    \rho(\sysOp, \sysOp') := \norm{\transF - \transF'}_{\Hardy{\infty}}.
\end{equation}
The following result relates $\rho(\sysOp, \sysOp')$ to distance---in terms of squared error---in the system output space.

\begin{theorem}\label{thm:dist-hardy-output}
    Let $\sysOp$ and $\sysOp'$ be LTI systems with corresponding transfer functions $K(z)$ and $K'(z)$, both in $\Hardy{\infty}$. 
    It holds that
    \[
        \rho(\sysOp, \sysOp') = \sup_{\norm{x}_{\ell^2} = 1} \norm{\sysOp x - \sysOp' x}_{\ell^2}.
    \]
    \begin{proof}
    The proof is established through the following chain of arguments
    \begin{align}
        \rho(\sysOp, \sysOp') &= \norm{\transF - \transF'}_{\Hardy{\infty}} \\
        &= \sup_{X\,\in\,\Hardy{2}} \frac{\norm{(\transF-\transF')X}_{\Hardy{2}}}{\norm{X}_{\Hardy{2}}} \label{teq:inf-l2-equiv}\\
        &=\sup_{\norm{X}_{\Hardy{2}} = 1} \norm{(\transF - \transF')X}_{\Hardy{2}} \\
           &= \sup_{\norm{x}_{\ell^2} = 1} \norm{\impResp* x - \impResp'* x}_{\ell^2} \label{teq:52}\\
           &= \sup_{\norm{x}_{\ell^2} = 1} \norm{\sysOp x - \sysOp' x}_{\ell^2},
       \end{align}
    where (\ref{teq:inf-l2-equiv}) follows from Theorem \ref{thm:opNorm_eq_infty} upon noting that $K-K'\,\in\,\Hardy{\infty}$ by application of the triangle inequality and (\ref{teq:52}) is by
    Theorem \ref{thm:z_isometry} together with (\ref{eq:conv_is_multiplication}), where $k$ and $k'$ denote the impulse responses of the systems $\mathcal{L}$ and $\mathcal{L}'$, respectively.
    \end{proof}
\end{theorem}
Theorem \ref{thm:dist-hardy-output} shows that identifying a reference system $\mathcal{L}$ to within error $\rho(\sysOp, \sysOp')=\epsilon$ guarantees that the estimated system $\mathcal{L}'$ results in output signals that deviate no more than $\epsilon$---in $\ell_2$-norm---from the output that would be produced by the reference system $\mathcal{L}$.

We are now ready to recall a result due to Zames and Owen \cite{zamesDisc} which quantifies the metric entropy of $\mathcal{C}(\mC{}, \eC{})$ with respect to the distance measure $\rho(\sysOp, \sysOp')$. 

\begin{theorem}[\cite{zamesDisc}] \label{zames-covering}
Let $\mC{}, \eC{} > 0$ and consider the set 
\[
    \mathcal{C}(\mC{}, \eC{})  = \{ \sysOp \mid |k_{\sysOp}[t]| \leq \mC{}e^{-\eC{}t},\, \forall t\geq 0\}.
\]
The metric entropy of $\mathcal{C}(\mC{}, \eC{})$ with respect to 
\begin{equation}
    \rho(\sysOp, \sysOp')  = \norm{\transF - \transF'}_{\Hardy{\infty}} \label{eq:sys-dist}
\end{equation}
satisfies
\begin{equation}\label{eq:zames_rate}
    \mEnt(\epsilon; \mathcal{C}(\mC{}, \eC{}), \rho) \thicksim \frac{1}{\eC{}}\left(\log\left(\frac{\mC{}}{\epsilon}\right)\right)^2.
\end{equation}

\end{theorem}

Note that for all systems $\mathcal{L} \, \in \, \mathcal{C}(\mC{}, \eC{})$, the transfer function is in $\Hardy{\infty}$, which by application of the triangle inequality, shows that (\ref{eq:sys-dist}) is well-defined. The proof of Theorem \ref{zames-covering} proceeds by establishing lower and upper bounds on metric entropy according to
   \begin{equation}\label{eq:lower_bound}
       \frac{1}{\eC}\left(\log\left(\frac{\mC}{\epsilon}\right)\right)^2 - o\left(\left(\log\left(\frac{\mC}{\epsilon}\right)\right)^2  \right)  \leq \entOfSetLazy 
  \end{equation}
  and
  \begin{equation}\label{eq:upper_bound}
        \entOfSetLazy  \leq \frac{1}{\eC}\left(\log\left(\frac{\mC}{\epsilon}\right)\right)^2 + o\left(\left(\log\left(\frac{\mC}{\epsilon}\right)\right)^2  \right), 
  \end{equation}
  respectively, where $g(\epsilon)=o\left(f(\epsilon)\right)$ stands for $\lim_{\epsilon \rightarrow 0 }\left|\frac{g(\epsilon)}{f(\epsilon)}\right| = 0$. 
  The lower bound (\ref{eq:lower_bound}) is derived through an embedding argument and the upper bound is obtained by constructing an explicit $\eps$-covering \cite{zamesDisc}.
  
  We note that the metric entropy in (\ref{eq:zames_rate}) scaling according to $(\log(1/\epsilon))^2$ 
  shows that the set $\mathcal{C}(\mC{},\eC{})$ is not overly massive. Richer function classes such as the set of all Lipschitz functions from $[0,1]^d$ to $\R$ with a given Lipschitz constant have metric entropy scaling according to $\epsilon^{-d}$ \cite{Wainwright2019}. 
  Moreover, it follows from (\ref{eq:zames_rate}) 
  that impulse responses of slower (exponential) decay, i.e., with smaller $\eC{}$, are more complex to describe.
  
\section{Optimal Covering Through Quantized RNNs}\label{sec:quant_rnn}

We are now in a position to state the second central result of this paper. Specifically, we show that RNNs with suitably quantized weights provide an optimal---in the sense of Theorem \ref{zames-covering}---$\epsilon$-covering of $\mathcal{C}(\mC{}, \eC{})$ with respect to the metric $\rho({\cal L},{\cal L}')$. Operationally, this means that RNNs can optimally---in the sense of metric entropy---learn (or identify) the class of LTI systems with exponentially decaying impulse response. This result quantifies what is possible in principle and thereby provides a benchmark against which practical learning algorithms can be assessed.

The results presented so far apply to RNNs with real-valued weights. Constructing an optimal covering through RNNs requires, however, encoding of the approximating RNNs into bitstrings of length scaling in the approximation error $\eps$ according to (\ref{eq:zames_rate}). Now, there are two components that go into such an encoding of RNNs, namely the values of the nonzero weights in the matrices $A_1,A_2$ and the vectors $b_1,b_2$ and the locations of these weights, i.e., the topology of the network. The former requires quantization of the weights at a resolution that scales adequately in $\epsilon$. We shall see below that encoding the topology is a non-issue. The main technical problem hence resides in ensuring that weight quantization in the approximating RNN can be effected at a resolution that allows metric entropy optimality---in terms of the covering realized---and at the same time guarantees that the resulting error incurred at the system output consorts with the desired approximation accuracy.

We start by defining an RNN weight quantization scheme.

\begin{definition}[Quantized weights] \label{def:quant}
    For $\delta>0$, define the set
    \begin{equation}
            \quantSet  := \{\delta k \mid k \in \mathbb{Z} \}.
        \end{equation}
    We say that an RNN has $\delta$-quantized weights if all its weights are in $\quantSet \, \cup \, \{-1,1\}$. Further, define the quantization
    function 
    \begin{equation}
        \quantF{w}:= sign(w)  \floor*{\frac{|w|}{\delta}}  \delta, \quad w \in \R.
    \end{equation}
    Clearly, we have $|\quantF{w} - w| \leq \delta$ and $|\quantF{w}| \leq |w|$.
\end{definition}

The main idea underlying the proof of the optimal RNN covering result builds on the approximation of the exponentially decaying impulse responses
in $\mathcal{C}(\mC{}, \eC{})$ through finite impulse response (FIR) filters of suitable length and with suitably quantized impulse response coefficients. In order to quantify the approximation error---in terms of $\rho(\sysOp, \sysOp') = \norm{\transF - \transF'}_{\Hardy{\infty}}$---resulting from this truncation and coefficient quantization, we will need the following
simple technical result.
\begin{lemma}\label{lem:hard_diff_l1_bound}
    Consider the LTI systems with impulse responses $k[\cdot]$ and $\widetilde{k}[\cdot]$ and corresponding transfer functions $K(z)$ and $\widetilde{K}(z)$, both in $\Hardy{\infty}$. We have
    \begin{equation*}
        \norm{K(\cdot) - \widetilde{K}(\cdot)}_{\Hardy{\infty}}  \leq \sum_{t=0}^{\infty} |k[t]-\widetilde{k}[t]|.
    \end{equation*}
    \begin{proof}
    The proof is by the following chain of relations 
    \begin{align*}
      \norm{K(\cdot) - \widetilde{K}(\cdot)}_\Hardy{\infty} &= 
      \sup_{|z|<1}\left | \sum_{t=0}^\infty k[t]z^{t} - \sum_{t=0}^\infty \widetilde{k}[t]z^{t} \right| \\
      & = \sup_{|z|<1}\left | \sum_{t=0}^\infty (k[t] - \widetilde{k}[t])z^{t} \right| \\
      &\leq \sup_{|z|<1}  \sum_{t=0}^\infty |k[t] - \widetilde{k}[t] | |z|^{t} \\
      &= \sum_{t=0}^\infty |k[t]- \widetilde{k}[t]|. &&\qedhere
     \end{align*}
    \end{proof}
\end{lemma}

We are now ready to state the main result.

\begin{theorem}[RNNs are metric-entropy-optimal]\label{thm:iir_quant_error}
  Consider an LTI system $\sysOp{}$ with impulse response satisfying $|k[t]| \leq \mC{}\,e^{-\eC{} t}, \, \forall t \geq 0$, for some $\mC{}, \eC{} > 0$, 
  and corresponding transfer function $K(z)$. 
  For every $\epsilon >0$, with
  \[ M:= \ceil*{\frac{1}{\eC{}}\log \left(\frac{\mC{}}{\epsilon} \right) + \frac{1}{\eC{}}\log \left(\frac{2}{1-e^{-\eC}} \right) }, \]  
  $\sysOp{}$ can be approximated by a $\delta := \frac{\epsilon}{2M}$-quantized RNN $\quantRnnOp{}$---of hidden state size $M-1$---realizing an FIR filter with transfer function $\widetilde{K}(z)$ such that
  \[
  \norm{K(\cdot) - \widetilde{K}(\cdot)}_{\Hardy{\infty}}  \leq \epsilon.
  \]
  Moreover, $\quantRnnOp{}$ can be encoded in a uniquely decodable fashion, provided that both encoder and decoder know $\mC$ and $\eC$, using no more than
    \[
        \frac{1}{\eC} \left(\log \left(\frac{\mC}{\epsilon} \right)\right)^2 + o\left( \left(\log \left(\frac{1}{\epsilon}\right)\right)^2 \right)
    \] bits.
\end{theorem}
\begin{proof}

  The idea of the proof is to $\delta$-quantize the suitably truncated impulse response $k[t]$ corresponding to $\mathcal{L}$, 
  which is then realized (exactly) by an RNN, denoted as $\quantRnnOp{}$, following the construction 
  in Lemma \ref{lem:rnn_lin_combi_of_history}. 
  Concretely, we choose
    the truncated quantized impulse response according to $\widetilde{k}[t] := \quantF{k[t]}\ind{t\leq (M-1)}$, denote the corresponding transfer function by $\widetilde{K}(z)$,
  and then use Lemma \ref{lem:hard_diff_l1_bound} to bound 
\begin{align}
\norm{K(\cdot) - \widetilde{K}(\cdot)}_{\Hardy{\infty}}  & \leq \sum_{t=0}^{\infty} |k[t]-\widetilde{k}[t]| \\
  &= \sum_{t=0}^{M-1} |k[t]-\widetilde{k}[t]| + \sum_{t=M}^{\infty}|k[t]| \label{eq:quant_iir_impulse_zero} \\
  &\leq M  \delta + \sum_{t=M}^{\infty}\mC{}e^{-\eC{} t} \label{eq:quant_iir_bounds}  \\
  &= M  \delta + \mC{}\, \frac{e^{-\eC{}M}}{1-e^{-\eC}} \label{eq:both-bounds} \\
  &\leq M  \delta + \mC{}\, \frac{e^{-\log\left(\frac{2\mC}{\epsilon(1-e^{-\eC})}\right)}}{1-e^{-\eC}}\\
  &=  M  \delta + \mC{}\, \frac{\frac{\epsilon(1-e^{-\eC})}{2\mC}}{1-e^{-\eC}} \\
  &=  M  \frac{\epsilon}{2M} + \frac{\epsilon}{2} = \epsilon, \label{eq:quant_iir_final}
\end{align}
where in (\ref{eq:both-bounds}) we used $\sum_{n=M}^{\infty} r^n$ 
$=\frac{r^M}{1-r}$, for $|r|<1$.

It remains to establish that the RNN realizing the FIR system with impulse response $\widetilde{k}[t]$ can be encoded in a uniquely decodable fashion into a bitstring of length consorting with 
covering optimality according to (\ref{eq:zames_rate}). 
    As mentioned earlier, encoding an RNN in a bitstring requires specifying its topology and quantized weights, both in binary form. We first convince ourselves that the topology of the RNN realizing $\widetilde{k}[t]$ is fixed and hence does not need to be encoded. This follows by recognizing that in the RNN construction in the proof of  Lemma \ref{lem:rnn_lin_combi_of_history} the quantities $A_1,A_h,A_r,b_1$, and $b_2$ are all independent of the impulse response of the FIR system to be realized and only $A_o$ depends on the impulse response according to $A_o=\tilde{k}^T$. The locations of the nonzero entries in the weight matrices and bias vectors of the approximating RNN hence need not be encoded. This leaves us with having to represent the $M$ quantized impulse response coefficients $\widetilde{k}[t]$ through a bitstring of length scaling in $\eps$ such that covering optimality is attained. To this end,
            we first note that from Definition \ref{def:quant}, we get 
        \[
        |\widetilde{k}[t]| = |\quantF{k[t]} |\leq |k[t]| \leq \mC{}e^{-\eC{} t}, \qquad \forall t \in \{0, \dots, M-1\}.
        \]
        The quantized impulse response coefficients hence satisfy 
        \[
        \widetilde{k}[t] \in \quantSet \, \cap \, [-\mC{}e^{-\eC{} t}, \mC{}e^{-\eC{} t}], \qquad \forall t \in \{0, \dots, M-1\},
        \]
        and can therefore be stored using at most $\left\lceil\log_2\left ( \frac{\mC{}e^{-\eC{} t}}{\delta} \right )\right \rceil + 1$ bits. As $\mC$ and $\eC$ are known to the encoder and the decoder by assumption, we can encode the quantized impulse response coefficients into a uniquely decodable bitstring simply by allocating $\left \lceil \log_2\left ( \frac{\mC{}e^{-\eC{} t}}{\delta} \right )\right\rceil + 1$ bits to each coefficient, concatenating the corresponding binary labels (filled up with zeros if they are of smaller than the alloted length) and have the decoder read out the labels sequentially to deliver the corresponding points in $\quantSet$. 
        
        It remains to establish that the length of the bitstring just constructed conforms with (\ref{eq:zames_rate}).
        To this end, we first upper-bound the length of the bitstring according to 
        \begin{align}
            \sum_{t=0}^{M-1} & \left( \left\lceil\log_2\left ( \frac{\mC{}e^{-\eC{} t}}{\delta} \right )\right \rceil + 1 \right) \\
            &\leq \sum_{t=0}^{M-1} \left( \log_2 \left ( \frac{\mC}{\delta}  \right ) + \log_2(e^{-\eC{}t}) + 2 \right)  \\
            &= 2M + M \logTe{} \log\left ( \frac{\mC}{\delta}  \right ) + \sum_{t=0}^{M-1} \logTe{} \log\left(e^{-\eC{}t}\right ) \\
            &= 2M + M \logTe{}  \log\left ( \frac{\mC}{\delta}  \right ) - \eC{} \logTe{} \sum_{t=0}^{M-1} t \\
            &= 2M+ M \logTe{}  \log\left ( \frac{2\mC{}M}{\epsilon}  \right ) - \eC{}  \logTe{}\, \frac{M(M-1)}{2} \label{teq:plugin_delta} \\
            &= M\left (  \logTe{}  \log\left ( \frac{2\mC{}M}{\epsilon}  \right ) - \eC{}  \logTe{}\, \frac{(M-1)}{2} + 2 \right )  \\
            & =  M\left (  \logTe{} \log\left ( \frac{\mC}{\epsilon}  \right )  - M \frac{\eC{} \logTe{} }{2} + \logTe{} \log(M)  + \frac{\eC{}\logTe{}}{2} + 3 \right ),  \label{teq:lastUB}
        \end{align}
        where we used $\log_2(x)= \logTe{} \log(x)$ with $\logTe{} := \log_2(e)$ and in (\ref{teq:plugin_delta}) we employed $\delta = \frac{\epsilon}{2M}$.
        Next, we note from the definition of $M$ that
        \begin{equation}
             \frac{1}{\eC{}} \log \left(\frac{\mC{}}{\epsilon}\right)  + K_1(\eC{}) \leq M \leq \frac{1}{\eC{}} \log \left(\frac{\mC{}}{\epsilon} \right)  + K_1(\eC{}) + 1, \label{eq:upper-and-lower-on-m}
        \end{equation}
        with $K_1(\eC{}) := \frac{1}{\eC} \log \left(\frac{2}{1-e^{-\eC}} \right)$.
        Using (\ref{eq:upper-and-lower-on-m}) in (\ref{teq:lastUB}) allows us to further
        upper-bound (\ref{teq:lastUB}) as follows:
        \begin{align}
             & M\left (  \logTe{} \log\left ( \frac{\mC{}}{\epsilon}  \right )  - M \frac{\eC{} \logTe{} }{2} + \logTe{} \log(M)  + \frac{\eC{} \logTe{}}{2} + 3 \right ) \\
             &\leq M\left (  \logTe{} \log\left ( \frac{\mC{}}{\epsilon}  \right )  -   \frac{1}{\eC{}} \log \left(\frac{\mC{}}{\epsilon}\right)\frac{\eC{} \logTe{} }{2}  - K_1(\eC{}) \frac{\eC{} \logTe{} }{2} + \logTe{}  \log(M)  + \frac{\eC{} \logTe{}}{2} + 3 \right ) \\
             &= M\left (  \frac{\logTe{} }{2} \log \left ( \frac{\mC{}}{\epsilon}  \right ) + \logTe{}  \log (M)  + K_2(\eC{})\right ) \\
             &\leq \left (\frac{1}{\eC{}} \log \left(\frac{\mC{}}{\epsilon} \right)  + K_1(\eC{}) + 1 \right) \left (  \frac{\logTe{} }{2} \log \left ( \frac{\mC{}}{\epsilon}  \right ) + \logTe{} \log (M)  + K_2(\eC{})\right ) \\
             &= \frac{\logTe{} }{2\eC{}} \left( \log \left(\frac{\mC{}}{\epsilon} \right) \right)^2 + \frac{\logTe{}}{\eC{}} \log \left(\frac{\mC{}}{\epsilon} \right)  \log (M) + \frac{1}{\eC{}} \log \left(\frac{\mC{}}{\epsilon} \right) K_2(\eC{}) \\
             &\quad +(K_1(\eC{})+1)\frac{\logTe{}}{2} \log \left ( \frac{\mC{}}{\epsilon}  \right ) + (K_1(\eC{})+1)\logTe{} \log (M) + (K_1(\eC{}) + 1 ) K_2(\eC{}) \\
             &= \frac{\logTe{} }{2\eC{}} \left( \log \left(\frac{\mC{}}{\epsilon} \right) \right)^2 + \frac{\logTe{}}{\eC{}}  \log (M) \log  \left(\frac{1}{\epsilon} \right) +  K_3(\eC{}) \log \left(\frac{1}{\epsilon} \right)\\
             &\quad +K_4(\mC{}, \eC{}) \log (M) + K_5(\mC{}, \eC{}) \\
             &\leq \frac{1}{\eC{}} \left( \log \left(\frac{\mC{}}{\epsilon} \right) \right)^2 + o\left( \left( \log  \left(\frac{1}{\epsilon}\right) \right)^2 \right),
        \end{align}
        where $K_2(\eC) :=  - K_1(\eC{}) \frac{\eC{} \logTe{} }{2}  + \frac{\eC{}\logTe{}}{2} + 3$, $K_3(\eC{}) := \frac{K_2(\eC{})}{\eC} + \frac{ (K_1(\eC{}) +1) \logTe{}}{2}$, $K_4(\mC{}, \eC{}) := \logTe{} (K_1(\eC{}) + 1) + \log (\mC{}) \frac{\logTe{}}{\eC{}}$, and $K_5(\mC{}, \eC{}) := (K_1(\eC{})+1)K_2(\eC{}) + K_3(\eC{}) \log(\mC{})$. 
        The last inequality follows from $\logTe{} \leq 2$ and $\log (M) = o(\log (\epsilon^{-1}))$.
    \end{proof}
We conclude by noting that the dependence of the hidden state size and the weight quantization resolution of the approximating RNN in Theorem \ref{thm:iir_quant_error} on the parameters $\mC{}, \eC{},\epsilon$ reflects that more complex sets $\mathcal{C}(\mC{}, \eC{})$ and smaller target approximation error require larger hidden state size and higher quantization resolution.

\section{Metric-Entropy-Optimal Learning of Linear Difference Equations} \label{sec:metric-entropy-learning-difference-equations}

Over the last few years a significant body of literature on deep neural network learning of the solutions of parametric PDEs was developed \cite{Grohs1,Grohs2,Raslan2021}. More specifically, this line of work is concerned with learning the map taking the right-hand side of the PDE and its parameters to the solution. We next suggest an alternative viewpoint 
in its simplest possible mathematical incarnation, namely that of learning
differential, in fact difference, equations themselves. 
From a practical perspective this amounts to identifying the dynamics of physical, biological, mechanical, or chemical processes from observed input-output traces \cite{Vlachas2018complex}.

We consider linear difference equations with constant coefficients given by
\begin{equation}
  \sum_{j=0}^{P} b_j y[t - j] = \sum_{i=0}^{Q} a_i x[t-i], \label{diff-eq-rational}
\end{equation}
where $P, Q \in \mathbb{N}, b_j,a_i \in \R$, and $x[t]$ and $y[t]$ designate the input and the output, respectively, of the dynamical system characterized by the difference equation. Difference equations of the form (\ref{diff-eq-rational}) correspond to LTI systems
with rational transfer functions. Concretely application of Lemma \ref{lem:time_shift} yields $Y(z)=K(z)X(z)$ with
\begin{equation}
    	K(z) = \frac{\sum_{i=0}^{Q} a_i z^i}{\sum_{j=0}^{P} b_j z^j}. \label{eq:iir_transfer_func}
\end{equation}
Learning of the difference equation (\ref{diff-eq-rational}) from input-output traces, i.e., determining the coefficients
$a_i$ and $b_j$ in (\ref{diff-eq-rational}) based on the outputs $y[\cdot]$ corresponding to given inputs $x[\cdot]$
hence amounts to identifying the LTI system with transfer function (\ref{eq:iir_transfer_func}). We first convince ourselves that RNNs can, in principle, realize systems with rational transfer functions, thereby extending Lemma \ref{lem:rnn_lin_combi_of_history} where this was shown for polynomial transfer functions.

\begin{theorem}[RNNs can realize all rational transfer functions]
\label{thm:iir_rational_construction}
    Let $\sysOp$ be an LTI system with transfer function 
    \begin{equation}
    	K(z) = \frac{\sum_{i=0}^{Q} a_i z^i}{\sum_{j=0}^{P} b_j z^j}, 
    \end{equation}
    where $Q, P \in \mathbb{N}$ and $a_i,b_j \in \R$ with $b_0 \neq 0$. Then, there exists an RNN that realizes $\sysOp{}$ exactly. 
    
    \begin{proof}
    The proof will be effected by constructing the RNN realizing $\sysOp$. We start by noting that application of the inverse $\mathcal{Z}$-transform and the time shift property Lemma \ref{lem:time_shift} to
\[
    Y(z) = X(z) K(z)
\]
yields the difference equation
\begin{equation}
y[t] = \sum_{i=0}^{Q} c_i x[t-i] + \sum_{j=1}^{P} d_j y[t - j], \label{eq:final_y_formula} 
\end{equation}
with $c_i=\frac{a_i}{b_0}$ and $d_j=-\frac{b_j}{b_0}$.
In contrast to the construction in Lemma \ref{lem:rnn_lin_combi_of_history} which realizes a
forward part only, here given by $\sum_{i=0}^{Q} c_i x[t-i]$, we will need to
account for both the forward part and the backward part $\sum_{j=1}^{P} d_j y[t - j]$.
This will be accomplished by 
choosing the RNN weight matrices $A_1,A_2$ and bias vectors $b_1,b_2$ such that
the hidden state vector $\vec{h}[t]$ contains the last $Q$ values of the input signal $x[\cdot]$ and the last $P$ values of the output signal $y[\cdot]$ according to
\begin{align}
	\h{t} &= 
    \NiceMatrixOptions
{nullify-dots,code-for-first-col = \color{blue},code-for-first-row=\color{blue}, code-for-last-col=\color{blue} }
	\begin{pNiceArray}{c}[first-col]
	\Vdotsfor[line-style={solid,<->},shorten=1pt]{3}_{\scriptscriptstyle Q} &
	x[t] \\
	&\vdots \\[1mm]
	&x[t-(Q-1)] 
	\\\hline
	\Vdotsfor[line-style={solid,<->},shorten=1pt]{3}_{\scriptscriptstyle P} &
	y[t] \\
	&\vdots \\[1mm]
	&y[t-(P-1)] 
	\end{pNiceArray}. 
	\label{eq:iir_h}
\end{align}
Then, based on (\ref{eq:iir_h}), we establish that these choices also yield the output signal as desired. We commence
by specifying the RNN weights and proving (\ref{eq:iir_h}) by induction. 
By slight abuse of
notation, we let the vector $\vec{c}$ have a zero-th entry and define
\begin{align}
    \begin{split}
	\vec{c} &:= \begin{pmatrix}	c_0 & c_1 & \dots & c_Q \end{pmatrix} ^T \in \R^{Q+1}, \\
	\vec{d} &:= \begin{pmatrix} d_1 & \dots & d_P	\end{pmatrix}^T  \in \R^{P}, \label{eq:iir_k}
    \end{split}
\end{align}
and the matrix 
\begin{equation}\label{teq:W_iir}
    W := \;
    \NiceMatrixOptions
{nullify-dots,code-for-first-col = \color{blue},code-for-first-row=\color{blue}, code-for-last-col=\color{blue} }
        \begin{pNiceArray}{cc|cc}[first-row, first-col]
        &\Hdotsfor[line-style={solid,<->}, shorten=1pt]{2}^{Q+1} & 
         \Hdotsfor[line-style={solid,<->}, shorten=1pt]{2}^{P} 
        \\
        \Vdotsfor[line-style={solid,<->},shorten=1pt]{1}_{\scriptscriptstyle 1} 
        & \Block{1-2}{\vec{c}^T}& &\Block{1-2}{\vec{d}^T} 
        \\
        \Vdotsfor[line-style={solid,<->},shorten=1pt]{1}_{\scriptscriptstyle Q} 
		& \Imat{Q}  & \Ovec{Q} & \Block{1-2}{\Omat} &
		\\
		\Vdotsfor[line-style={solid,<->},shorten=1pt]{1}_{\scriptscriptstyle 1} 
	    & \Block{1-2}{\vec{c}^T}& &\Block{1-2}{\vec{d}^T} 
		\\
		\Vdotsfor[line-style={solid,<->},shorten=1pt]{1}_{\scriptscriptstyle P-1} 
		& \Block{1-2}{\Omat} &  &\Imat{P-1} & \Ovec{P-1}
         \end{pNiceArray},
\end{equation}
where the unsubscripted symbols $\Omat$ stand for all-zeros matrices of appropriate dimensions.
The network weights are now chosen according to
\begin{align}\label{teq:weights_iir}
    A_1 & = \begin{pmatrix}
        \Imat{P+Q+1} \\[1mm]
        -\Imat{P+Q+1} \\
    \end{pmatrix},
    \qquad
    A_2  = W\begin{pmatrix}
    \Imat{P+Q+1} & - \Imat{P+Q+1}
    \end{pmatrix},
\end{align}
and $b_1=0_{2P+2Q+2}$, $b_2=0_{P+Q+1}$. With (\ref{eq:weights}) and thanks to (\ref{teq:relu_identiy}), this yields
\begin{equation}
	\begin{pmatrix}	\rnnOut[t] \\\h{t}	\end{pmatrix}  = W \begin{pmatrix}x[t]  \\ \h{t-1}  \end{pmatrix}, \quad \forall t \geq 0. \label{eq:lin_map_rnn_iir}
\end{equation}
We are now ready to establish (\ref{eq:iir_h}) by induction. First, we note that for $h[t]$ in (\ref{eq:iir_h}) to constitute a valid hidden state sequence according to Definition \ref{def:rnn}, the initial state needs to satisfy $h[-1]=0_{Q+P}$. This, indeed, follows from the assumption $x[t]=y[t]=0,\, \forall t < 0$, and, in turn, also yields the base case $t=-1$ of the induction argument.
To establish the induction step, we assume that (\ref{eq:iir_h}) holds for $t-1$ for some $t \geq 0$. Next,
let $\vec{h}_{1:Q}[t]\in\R^Q$ denote the subvector of $\vec{h}[t]$ containing the entries $1$ through $Q$. It now follows from (\ref{eq:lin_map_rnn_iir}) and (\ref{teq:W_iir}) that
\[
    \vec{h}_{1:Q}[t]  =
    \begin{pmatrix} \Imat{Q} & \Ovec{Q} \end{pmatrix} \begin{pmatrix} 
    x[t]\\
    x[t-1]\\
    \vdots \\[1mm]
    x[t-Q] 
    \end{pmatrix}
    = \begin{pmatrix}
    x[t]\\
    \vdots\\[1mm]
    x[t-(Q-1)]
    \end{pmatrix}
\]
and 
\begin{align}\label{eq:output-signal}
\begin{split}
    h_{Q+1}[t]  
    &= \begin{pmatrix}\vec{c}^T & \vec{d}^T \end{pmatrix} 
    \begin{pmatrix}x[t] \\ \vec{h}[t-1] \end{pmatrix}\\
     &= \begin{pmatrix}\vec{c}^T & \vec{d}^T \end{pmatrix}
     \begin{pmatrix}x[t] \\ \vdots \\[1mm] x[t-Q] \\ 
     y[t-1] \\ \vdots \\[1mm] y[t-P] \end{pmatrix}
    \\
    &= \sum_{i=0}^{Q} c_i x[t-i] + \sum_{j=1}^{P} d_j y[t-j]=y[t],
    \end{split}
\end{align}
where we used (\ref{eq:final_y_formula}).
The proof of the induction step is now completed upon noting that
\begin{align*}
    \vec{h}_{(Q+2):(Q+P)}[t] &= W\begin{pmatrix}x[t] \\ \vec{h}[t-1] \end{pmatrix} \\
    &= \begin{pmatrix} \Imat{P-1} & \Ovec{P-1} \end{pmatrix} 
    \begin{pmatrix} y[t-1] \\ \vdots \\[1mm] y[t-P] \end{pmatrix}\\
    &= \begin{pmatrix} y[t-1] \\ \vdots \\[1mm] y[t-(P-1)] \end{pmatrix}.
\end{align*}
Finally, it follows by combining (\ref{eq:lin_map_rnn_iir}), (\ref{teq:W_iir}), and (\ref{eq:output-signal}) that the weights we chose yield the desired output signal.
\end{proof}
\end{theorem}

We have hence established that RNNs with real-valued weights can realize LTI systems with rational transfer functions exactly. 
In fact, as inspection of the weight matrices $A_1,A_2$ and the bias vectors $b_1,b_2$ in the proof of Theorem \ref{thm:iir_rational_construction} reveals, the size of the RNN is $\mathcal{O}(P+Q)$ and hence 
proportional to the number of parameters in the system transfer function.

We now proceed to argue that the results established in Section \ref{sec:quant_rnn} provide a fundamental limit on how well difference equations of the form (\ref{diff-eq-rational}) can be learned
in principle and that RNNs can achieve this fundamental limit. But first, we state an important restriction, namely to LTI systems (of rational transfer function) that have corresponding impulse responses in $\ell_1$. In system theory parlance such systems are often referred to as stable \cite[Section 2.6]{kailath1980linear}.
If the coefficients of $K(z)$ in (\ref{diff-eq-rational}) are such that the system is, indeed, stable, the impulse response is necessarily a linear combination of terms of the form $p(t)\cos(\modulationC{} t + \omega) \expBaseC{}^{t}$, where $p(t)$ is a polynomial in $t$, $\expBaseC{} \in (0, 1)$, and $\modulationC{},\omega \in \R$ \cite{Oppenheim2009}. Denoting the largest $\expBaseC{}$ occuring in this linear combination by $\widetilde{\expBaseC{}}$, this class of impulse responses is contained in the set $\mathcal{C}(\mC{}', \log(1/\expBaseC{}^{\prime}))$ with $\expBaseC{}^{\prime}\,>\,\widetilde{\expBaseC{}}$ and $\mC{}'$ chosen suitably, where such a $\expBaseC{}^{\prime} \in (0,1)$ always exists thanks to the set $(0,1)$  
being open and $\mC{}'$ exists as $\mC{}>0$ in (\ref{eq:exp-decay-general-exponent}). Application of \cite[Theorem~13.6]{Rudin1987} or inspection of the embedding 
argument used in \cite{zamesDisc} to establish the lower bound (\ref{eq:lower_bound}) shows that the covering number of the set of stable rational transfer functions equals that of $\mathcal{C}(\mC{}',\log(1/\expBaseC{}'))$. Hence, Theorem \ref{thm:iir_quant_error} allows us to conclude that RNNs can, in principle, learn difference equations of the form (\ref{diff-eq-rational}) with coefficients $a_i,b_j$ such that the corresponding LTI system is stable in a metric-entropy-optimal manner. 

\section{Conclusion}

The setting in this paper was deliberately chosen so as to allow the minimum level of mathematical sophistication needed to bring out the main conceptual findings. Numerous extensions abound, such as the continuous-time case and the approximation of nonlinear systems. It would furthermore be interesting to understand how metric entropy results can be obtained for linear time-varying systems. This would possibly allow to establish RNN covering optimality for general linear dynamical systems. From a control theory perspective our findings state that RNNs can be trained to optimally---in the sense of metric entropy---identify LTI systems. Here, it would be interesting to understand whether the presence of feedback, which is known to reduce identification complexity, could be incorporated into our theory and whether the corresponding fundamental limits can again be shown to be achievable through identification by RNNs. Furthermore, 
we consider it worthwhile to investigate how
concepts such as controllability, reachability, and observability for linear dynamical systems transfer to the state-space representation of RNNs realizing these systems.
An issue we have not touched upon at all is that of algorithms for learning the weights of approximating RNNs and whether such algorithms are likely to find the RNN constructions we exhibit. Another important aspect we did not discuss is that of minimality of linear dynamical system realizations \cite{kailath1980linear} and how it relates to corresponding RNN realizations \cite{DealingWithComplexity1998}. A question cast in the same mould is that of uniqueness of neural network realizations, a large field of research, both in feed-forward as well as recurrent neural network theory \cite{Fefferman94, Vlacic2021, Vlacic-adv-2021, nnUniquenessAlbertini93, Albertini93uniquenessof, StateObsSontag94}. 
Finally, we find that extensions of the ideas in Section \ref{sec:metric-entropy-learning-difference-equations} to linear, nonlinear, partial, and stochastic differential equations constitute a worthwhile endeavor. In this regard, we mention that the universal realization result Lemma \ref{lem:rnn_ltv_construction}, in its continuous-time incarnation, suggests that pseudo-differential operators \cite[Chapter~14]{Grochenig2001} can be represented exactly by (continuous-time) RNNs.

\appendix
\renewcommand{\thesection}{\Alph{section}}

\section{Alternative Definitions of RNNs}\label{app:elmanrnn}
\begin{definition}[Elman RNN] \cite{elman90}, \cite[p274]{Goodfellow2016}
    For $\stateDimElm\in\N$, weights $\uElm \in\R^{\stateDimElm \times 1}, \wElm_1 \in \R^{\stateDimElm \times \stateDimElm}$, $\wElm_2 \in \R^{1\times \stateDimElm}$, and biases $\bElm_1 \in \R^{\stateDimElm}$, $\bElm_2 \in \R$, the Elman RNN with hidden state sequence $\hidElm[t]\,\in\,\R^{\stateDimElm}$ of initial state $\hidElm[-1]=0_{\stateDimElm}$ and output $\rnnOut[t]\,\in\,\R$, for all $t\geq 0$, is defined by 
    \begin{align}
        \hidElm[t] &= \relu(\uElm x[t] + \wElm_1 \hidElm[t-1] + \bElm_1) \label{eq:elman_hidden}\\
        \rnnOut[t] &= \wElm_2 \hidElm[t] + \bElm_2. \label{eq:elman_out}
    \end{align}
\end{definition}

\begin{lemma}\label{lem:elman_def_same}
  The input-output relation of every RNN according to Definition \ref{def:rnn} can equivalently be realized by an Elman RNN.
   \begin{proof}
  Given an RNN according to Definition \ref{def:rnn} with weight matrices $A_1, A_2$ and bias vectors $b_1, b_2$, we construct an Elman RNN that realizes the same input-output map. First, set 
  \begin{align*}
     A_1 = \begin{pmatrix}
        A_x & A_g
      \end{pmatrix},
      \qquad
      A_2=
       \begin{pmatrix}
        A_y \\ A_h
      \end{pmatrix},
      \qquad
      b_1 = b_g,
      \qquad
      b_2=
      \begin{pmatrix}
        b_y \\ b_h
      \end{pmatrix},
  \end{align*}
  with $A_x\,\in\,\R^{n\times 1}$, $A_g\,\in\,\R^{n\times m}$, $A_y\,\in\,\R^{1\times n}$, $A_h\,\in\,\R^{m\times n}$, $b_g\in\R^{n}$, $b_y\,\in\,\R$, and $b_h\,\in\,\R^{m}$. With these definitions, (\ref{eq:weights})
  and (\ref{eq:rnn_seq}) can be written as
        \begin{equation}
      \begin{pmatrix}\label{teq:b5}
        \rnnOut[t] \\ h[t]
      \end{pmatrix}
      = \begin{pmatrix}
        A_y \\ A_h
      \end{pmatrix} g[t] + 
      \begin{pmatrix}
        b_y \\ b_h
      \end{pmatrix},
  \end{equation}
  where
  \[
      g[t] = \relu\left( 
      \begin{pmatrix}
        A_x & A_g
      \end{pmatrix} \begin{pmatrix}
        x[t] \\ h[t-1]
      \end{pmatrix}
      + b_g
      \right).
      \]

  The equivalent---in the sense of input-output relation---Elman RNN is now obtained by setting $\stateDimElm=n$ and
  \begin{equation}
  \begin{gathered}\label{teq:elman_weights}
    \wElm_1 = A_g A_h,  \qquad \uElm = A_x, \qquad \bElm_1 = A_g b_h + b_g, \\
    \wElm_2 = A_y, \qquad \bElm_2=b_y.
  \end{gathered}
  \end{equation}
  We first establish, by induction, that these choices lead to the hidden state sequences of the original RNN and the equivalent Elman RNN to be related according to $h[t] = A_h \hidElm[t] + b_h$, for all $t \geq 0$. The base case follows by choosing the initial hidden state $\hidElm[-1]$ of the Elman RNN such that
    $h[-1]= 0_{m} = A_h \hidElm[-1] + b_h$.
    If $b_h=0$, which is the case for all RNN constructions in this paper, one can, indeed, simply set $\hidElm[-1]=0$. 
  Next, we assume that
  $h[t-1] = A_h \hidElm[t-1] + b_h$, for some $t \geq 0$, and insert (\ref{teq:elman_weights}) into (\ref{eq:elman_hidden}) to obtain
  \begin{align*}
      \hidElm[t] &= \relu( A_x x[t] + A_g A_h \hidElm[t-1] + A_g b_h + b_g )\\
      &= \relu( A_x x[t] + A_g (A_h \hidElm[t-1] + b_h) + b_g )\\
      &= \relu( A_x x[t] + A_g h[t-1] + b_g )\\
      &= g[t]. \label{teq:hidElm_is_g}
  \end{align*}
  Using $\hidElm[t]=g[t]$ in (\ref{teq:b5}) then yields $h[t] = A_h \hidElm[t] + b_h$ as desired. This completes the proof of the induction step.
  The input-output relation of the Elman RNN is seen to equal that of the original RNN---given by (\ref{teq:b5}) as $y[t]=A_y g[t]+b_y$ ---upon inserting $\hidElm[t]=g[t]$, $\wElm_2=A_y$, and $\bElm_2=b_y$ in (\ref{eq:elman_out}).
  \end{proof}
\end{lemma}

\section{Properties of the \texorpdfstring{$\mathcal{Z}$}{Z}-transform and of Hardy Norms}\label{app:zproperties}
\begin{lemma}\label{lem:time_shift}
Let $x[t]$ be a one-sided sequence, i.e., $x[t]=0$, for $t<0$. Then, for $k\in\N$, it holds that
\[
    (\Z{x[\cdot-k]})(z) = z^k (\Z{x[\cdot]})(z).
\]
\begin{proof}
We have 
  \begin{align*}
    (\Z{x[\cdot-k]})(z) &= \sum_{t=0}^{\infty} x[t-k] z^t 
    = \sum_{t=-k}^{\infty} x[t]z^{t+k} \\
    & = z^k \sum_{t=0}^{\infty} x[t]z^{t} 
    = z^k (\Z{x[\cdot]})(z),
  \end{align*}
  where we used that $x[t]$ is one-sided.
 \end{proof}
\end{lemma}

\begin{theorem} \label{thm:z_isometry}
Let $x\in \ell^2$ be a one-sided sequence, i.e., $x[t]=0$, for $t<0$. Then,
we have
\[
    \norm{X}_{\Hardy{2}} = \norm{x}_{\ell^2}.
\]
\begin{proof}
    \begin{align*} 
        \norm{X}_{\Hardy{2}}^2 &= \sup_{r<1}\, \frac{1}{2\pi} \int_0^{2\pi} \left| X (re^{i\theta}) \right |^2 d\theta \\
         &= \sup_{r<1}\,\frac{1}{2\pi} \int_0^{2\pi} \left| \sum_{t=0}^{\infty} x[t] (re^{i\theta})^t \right |^2 d\theta \\
         &= \sup_{r<1} \,  \sum_{t=0}^{\infty} \sum_{t'=0}^{\infty} x[t]  \overline{x[t']} \, r^{t+t'} \frac{1}{2\pi}\int_0^{2\pi} e^{i\theta (t - t')} d\theta \\
         &= \sup_{r<1} \,  \sum_{t=0}^{\infty} \sum_{t'=0}^{\infty} x[t]  \overline{x[t']} \, r^{t+t'} \ind{t=t'} \\
         &= \sup_{r<1} \,  \sum_{t=0}^{\infty} |x[t]|^2  \, r^{2t} \\
         &= \sum_{t=0}^{\infty} |x[t]|^2 = \norm{x}_{\ell^2}^2. &&\qedhere
    \end{align*}
\end{proof}
\end{theorem}

\begin{theorem}\label{thm:opNorm_eq_infty}
For $K(\cdot)$ such that $\hNorm{\transF} < \infty$, it holds that
\begin{equation}
     \hNorm{\transF} = \sup_{X\in \Hardy{2}} \frac{\hhNorm{\transF X}}{\hhNorm{X}}. \label{Rochberg}
\end{equation}
\begin{proof}
The proof essentially follows \cite{McCarthy03} with minor refinements and details filled in.
We start by noting that the RHS of (\ref{Rochberg}) is the operator norm $\opNorm{K}:=\sup_{X\in \Hardy{2}} \frac{\hhNorm{KX}}{\hhNorm{X}}$ of the multiplication operator $X(z) \rightarrow K(z) X(z)$ and first establish that $\opNorm{K} \leq \hNorm{\transF}$. For every $X\in \Hardy{2}$,
we have
        \begin{align*}
        \norm{K\, X}_\Hardy{2} &= \HardyTwo{|K(re^{i\theta})  X(re^{i\theta})|^2} \\
        & \leq \HardyTwo{|X(re^{i\theta})|^2  \left(\sup_{|z|<1} |K(z)| \right)^2} \\
        & = \norm{K}_\Hardy{\infty}  \HardyTwo{|X(re^{i\theta})|^2} \\
        &= \norm{K}_\Hardy{\infty} \norm{X}_\Hardy{2},
      \end{align*}
which, upon division by $\norm{X}_\Hardy{2}$ establishes the desired upper bound.

To complete the proof, we show that $\opNorm{K} \geq \hNorm{\transF}$.
Applying 
\[\hhNorm{\transF X} \leq \opNorm{\transF} \hhNorm{X}
\]
repeatedly, we get, for every $n\in \N$,
\begin{align}\label{teq:22}
    \hhNorm{K^n X} \leq \opNorm{K}^n \hhNorm{X}.
\end{align}
Without loss of generality, we can restrict ourselves to $\opNorm{\transF}=1$ as otherwise we can simply consider $\transF' := \transF / \opNorm{\transF}$. 
Next, towards a contradiction, assume that $\opNorm{\transF} < \hNorm{\transF}$, which, thanks to $\opNorm{\transF}=1$, results in
$1< \hNorm{\transF} = \sup_{r<1, \, 0\, \leq\, \theta < 2\pi} |\transF(re^{i\theta})|$. As $\hNorm{\transF}<\infty$ by assumption, it follows that $\transF(z)$ is analytic and thus continuous inside the unit disk. Hence, there exist $0<r'<1, \epsilon>0$ and an interval $[\underline{\theta}, \overline{\theta})\,\in\,[0,2\pi)$
with $ \overline{\theta} - \underline{\theta}= \delta >0$ such that
\begin{equation}\label{teq:a24}
|\transF(r'e^{i\theta'})| > 1 + \epsilon, \quad \forall \theta' \in [\underline{\theta}, \overline{\theta}).
\end{equation}
Now we take $X(z)=1$ 
which clearly satisfies $\hhNorm{X}=1$. Inserting this into (\ref{teq:22}), we obtain
\begin{equation*}
     \hhNorm{\transF^n X}^2 \leq \opNorm{\transF}^{2n}\hhNorm{X}^2 = 1.
\end{equation*}
This, however, finalizes the proof by leading to the following contradiction
\begin{align}
    1 &\geq \hhNorm{\transF^n X}^2 \nonumber \\
    &= \sup_{0<r<1} \frac{1}{2\pi} \int_0^{2\pi} |\transF(re^{i\theta})|^{2n} d\theta \nonumber\\
    & \geq \frac{1}{2\pi} \int_0^{2\pi} |\transF(r'e^{i\theta})|^{2n} d\theta \nonumber \\
    & \geq \frac{1}{2\pi} \int_0^{2\pi} ((1+\epsilon)\ind{\theta \in [\underline{\theta}, \overline{\theta})})^{2n} d\theta \label{teq:a27}\\
    &= \frac{\delta}{2\pi} (1+\epsilon)^{2n} \xrightarrow[n\rightarrow \infty]{} \infty, \nonumber
\end{align}
where in
(\ref{teq:a27}) we used (\ref{teq:a24}) and the fact that $|\transF(r'e^{i\theta})| \geq 0$, for $\theta \notin [\underline{\theta}, \overline{\theta})$.
\end{proof}
\end{theorem}

\bibliographystyle{custom-els}
\bibliography{main}

\end{document}